\documentclass[12pt]{article}
\usepackage[utf8]{inputenc}
\usepackage{cite}
\usepackage{amsfonts}
\usepackage{amsmath}
\usepackage{amsthm}
\usepackage{enumitem}
\usepackage{amssymb}
\usepackage{caption}
\usepackage{graphicx}
\usepackage{mathrsfs}
\usepackage{multirow}
\usepackage{xcolor}
\usepackage{setspace}
\usepackage{booktabs}
\usepackage{latexsym, mathdots, array, extarrows}
\usepackage{algpseudocode,float}
\usepackage{lipsum}
\usepackage{pifont}
\usepackage{diagbox}

\usepackage{amsmath,amsfonts,amsthm,amssymb}
\usepackage{bm,dsfont}
\usepackage{booktabs}
\usepackage{graphicx}
\usepackage{array}
\usepackage{xcolor}
\usepackage{sidecap}
\usepackage{listings}
\usepackage{times}
\usepackage{url}
\usepackage{hyperref}
\usepackage{subfigure}
\usepackage{makecell} 

\usepackage{algorithm}      
\usepackage{algpseudocode}  

\usepackage{cleveref}

\usepackage{geometry}
\newtheorem{definition}{Definition}
\newtheorem{theorem}{Theorem}

\newtheorem{lemma}{Lemma}

\newcommand{\ba}{\bm{a}}
\newcommand{\bc}{\bm{c}}
\newcommand{\bG}{\bm{G}}
\newcommand{\bD}{\bm{D}}

\newcommand{\bH}{\bm{H}}

\newcommand{\bx}{\bm{x}}
\newcommand{\bU}{\bm{U}}

\newcommand{\bA}{\bm{A}}
\newcommand{\bB}{\bm{B}}
\newcommand{\bu}{\bm{u}}
\newcommand{\by}{\bm{y}}

\newcommand{\bW}{\bm{W}}
\newcommand{\bV}{\bm{V}}
\newcommand{\bv}{\bm{v}}
\newcommand{\blam}{\bm{\lambda}}
\newcommand{\bSig}{\bm{\Sigma}}
\newcommand{\bsig}{\bm{\sigma}}
\newcommand{\Sl}{\mathcal{S}_{\ell}}

\allowdisplaybreaks[4]

\title{Automatic Rank Determination for Low-Rank Adaptation via Submodular Function Maximization
}

\author{Yihang Gao\thanks{Department of Mathematics, National University of Singapore. Email: gaoyh@nus.edu.sg}\and Vincent Y. F.  Tan\thanks{Department of Mathematics and Department of Electrical and Computer Engineering, National University of Singapore. Email: vtan@nus.edu.sg}}

\date{}

\begin{document}

\maketitle

\begin{abstract}
In this paper, we propose SubLoRA, a rank determination method for Low-Rank Adaptation (LoRA) based on submodular function maximization. In contrast to prior approaches, such as AdaLoRA, that rely on first-order (linearized) approximations of the loss function, SubLoRA utilizes second-order information to capture the potentially complex loss landscape by incorporating the Hessian matrix. We show that the linearization becomes inaccurate and ill-conditioned when the LoRA parameters have been well optimized, motivating the need for a more reliable and nuanced second-order formulation.
To this end, we reformulate the rank determination problem as a combinatorial optimization problem with a quadratic objective. However, solving this problem exactly is NP-hard in general. To overcome the computational challenge, we introduce a submodular function maximization framework and devise a greedy algorithm with approximation guarantees. We derive a sufficient and necessary condition under which the rank-determination objective becomes submodular, and construct a closed-form projection of the Hessian matrix that satisfies this condition while maintaining computational efficiency.
Our method combines solid theoretical foundations, second-order accuracy, and practical computational efficiency. We further extend SubLoRA to a joint optimization setting, alternating between LoRA parameter updates and rank determination under a rank budget constraint. 
Extensive experiments on fine-tuning physics-informed neural networks (PINNs) for solving partial differential equations (PDEs) demonstrate the effectiveness of our approach. Results show that SubLoRA outperforms existing methods in both rank determination and joint training performance.
\end{abstract}

\section{Introduction}
Low-rank adaptation (LoRA)~\cite{hu2022lora} has demonstrated promising performance as an efficient fine-tuning technique across a wide range of domains, including large language models~\cite{zhang2024personalized,liu2024alora,valipour2023dylora,ding2023parameter}, vision models~\cite{frenkel2024implicit,chen2022adaptformer,liang2025lorasculpt,jia2022visual,zhang2024neural,awais2025foundation}, and physics-informed neural networks~\cite{majumdar2023pihlora,majumdar2023hyperlora,wang2025transfer}. Rather than updating all model parameters during fine-tuning, LoRA introduces a parameter-efficient strategy by applying low-rank decomposition to selected weight matrices.
The standard formulation of LoRA adopts classical Burer–Monteiro factorization to parameter updates. To better capture structural dependencies in model weights, Edalati et al.~\cite{edalati2022krona} propose using Kronecker decomposition for the low-rank parameter updates. Tensorized LoRA further generalizes this idea by reshaping parameter matrices into higher-order tensors and applying low tensor-rank approximations. 
Depending on the underlying model architectures~\cite{gao2025low,ding2024lora,tjandra2018tensor} and the inherent structure of the data~\cite{bershatsky2024lotr,yang2024loretta}, different tensor decomposition techniques can be utilized to exploit the structural efficiency.  
A broader review of recent developments in LoRA, including algorithmic enhancements, theoretical analyses, and various applications, can be found in~\cite{zeng2024the,jang2024lora,wang2024blob,hayou2024lora+,han2024parameterefficient,mao2025survey,wang2025parameter,you2022ranking,dettmers2023qlora,schmirler2024fine,kim2025lora}.

However, rank determination in LoRA remains an underexplored but critical aspect of its effectiveness. Despite its importance, relatively few works have addressed this problem in depth. 
Most existing methods rely on singular value decomposition (SVD) to factorize parameter updates, motivated by the fundamental relationship between singular values and matrix rank. Specifically, when a singular value approaches zero, the corresponding LoRA component contributes negligibly and can be discarded.
Broadly, two main strategies have been proposed. The first involves Bayesian inference. 
Yang et al.~\cite{yang2024bayesian} treat singular values as random variables and estimate their posterior distributions under a Bayesian framework. The rank budget is then allocated based on the total allowed rank and the inferred significance of each singular value.
However, this method introduces considerable computational overhead and suffers from high sensitivity to sampling strategies and prior assumptions.
A more computationally efficient alternative considers the importance of each singular value through loss sensitivity analysis. 
AdaLoRA~\cite{zhang2023adaptive}, for instance, linearizes the loss function with respect to the singular values of the parameter updates. The first-order approximation provides a sensitivity measure that guides adaptive pruning by identifying which components of the singular values contribute significantly to the loss.

Our work builds upon the direction of AdaLoRA by determining the rank allocation for each LoRA layer based on the importance of the singular values in the parameter updates. 
However, we observe that first-order linearization suffers from poor approximation accuracy, particularly when the LoRA fine-tuning reaches a stationary point. In such cases, the gradient vanishes, making linear sensitivity metrics unreliable for rank determination.
Motivated by this limitation, we propose to adopt a second-order expansion of the objective by incorporating Hessian information, enabling a more accurate approximation that captures the curvature of the loss landscape. 
This leads us to formulate the rank determination problem as a combinatorial optimization problem with set-valued quadratic objective.
Although the second-order formulation offers a better geometric understanding of the objective, it introduces an NP-hard optimization challenge, unlike the linearized formulation, which admits closed-form solutions.
To address this, we draw inspiration from the theoretical guarantees of the greedy algorithm in submodular function maximization. 
We introduce a Hessian projection that transforms the original set-valued quadratic objective into a submodular function. We prove that this projection is well-defined, admits a closed-form solution, and is computationally efficient.
As a result, the projected second-order objective becomes submodular, allowing us to apply greedy algorithms to obtain provably near-optimal solutions.
We apply the proposed technique to LoRA fine-tuning in physics-informed neural networks (PINNs) under rank budget constraints, where the loss landscape is especially complex and curvature information is valuable.
Our main contributions are summarized as follows:
\begin{enumerate}[label=(\roman*)]
    \item We show that linearization-based rank determination becomes unreliable near stationary points. To address this, we incorporate Hessian information and reformulate the problem as a combinatorial optimization problem with a set-valued quadratic objective.

    \item We design a projection of the Hessian matrix that transforms the objective into a submodular function, enabling efficient optimization by greedy algorithms with theoretical guarantees.

    \item We develop an alternating algorithm that jointly updates LoRA parameters and refines the rank allocation iteratively.

    \item We demonstrate the effectiveness of the proposed method on solving a class of PDEs using LoRA-based fine-tuning of PINNs. Experimental results show the superiority of the proposed rank determination method over other counterparts.
\end{enumerate}

The remainder of this paper is organized as follows.
In \Cref{sec_preliminary}, we introduce the notation and provide background on LoRA, rank determination, and submodular functions in combinatorial optimization.
\Cref{sec_method} presents our core methodology, where we reformulate the rank determination for LoRA as a submodular function maximization problem. In \Cref{sec_applications}, we demonstrate the application of the proposed method to LoRA fine-tuning in PINNs. Experimental results are presented and analyzed in \Cref{sec_experiments}. Finally, we conclude the paper with a discussion of key contributions, observations, and future directions in \Cref{sec_conclusion}.

\section{Preliminaries}
\label{sec_preliminary}
In this section, we first introduce the notation used throughout the paper. We then review the fundamental concepts of submodular functions in combinatorial optimization and provide a brief overview of LoRA. Finally, we discuss existing methods of rank determination for LoRA and highlight their potential limitations.

\subsection{Notation}
In this paper, we use bold lowercase letters (e.g., $\bx$) to denote vectors and bold uppercase letters (e.g., $\bA$) to represent matrices.
Scalars are represented using regular (non-bold) lowercase letters (e.g., $a$). 
The operator $\text{diag}(\cdot)$ constructs a square diagonal matrix from a given vector. Specifically, for $\ba \in \mathbb{R}^n$, we define $\bA = \text{diag}(\ba) \in \mathbb{R}^{n \times n}$ such that $A_{i,i} = a_i$ and $A_{i,j} = 0$ for $i \neq j$.
We denote the set $\{1,2,\cdots,n\}$ by $[n]$. Calligraphic letters (e.g., $\mathcal{S}$) are used to denote sets, and $|\mathcal{S}|$ denotes the cardinality (i.e., the number of elements) of set $\mathcal{S}$. For a given vector $\ba \in \mathbb{R}^n$ and a set $\mathcal{S} \subseteq [n]$, we define $[\ba]_{\mathcal{S}} \in \mathbb{R}^{|\mathcal{S}|}$ as the subvector of $\ba$ that contains only the entries indexed by $\mathcal{S}$. Similarly, for a matrix $\bA \in \mathbb{R}^{n \times n}$, we define $[\bA]_{\mathcal{S}} \in \mathbb{R}^{|\mathcal{S}| \times |\mathcal{S}|}$ as the principal submatrix of $\bA$ formed by retaining only the rows and columns indexed by $\mathcal{S}$. For a matrix $\bB \in \mathbb{R}^{m \times n}$, We define $(\bB)_{\mathcal{S}} \in \mathbb{R}^{m \times |\mathcal{S}|}$ as the column submatrix of $\bB$, constructed by selecting the columns indexed by $\mathcal{S}$.

\subsection{Submodular Function}
We consider a set-valued function $f: 2^{\Omega} \to \mathbb{R}$, where $2^{\Omega}$ denotes the power set of the finite ground set $\Omega$. The function $f$ is said to be submodular if it satisfies the following diminishing returns property:

\begin{definition}[See~\cite{krause2014submodular,fujishige2005submodular}]
\label{def_submodular}
    A set-valued function $f: 2^{\Omega} \to \mathbb{R}$ is submodular if for any sets $\mathcal{X}$, $\mathcal{Y} \subseteq \Omega$ with $\mathcal{X} \subseteq \mathcal{Y}$ and every element $x \in \Omega \setminus \mathcal{Y}$, the following holds:
    \begin{equation*}
        f(\mathcal{X} \cup \{x\}) - f(\mathcal{X}) \geq f(\mathcal{Y} \cup \{x\}) - f(\mathcal{Y}).
    \end{equation*}
    Equivalently, submodularity can be characterized by the following inequality:
    \begin{equation*}
        f(\mathcal{X}) + f(\mathcal{Y}) \geq f(\mathcal{X} \cup \mathcal{Y}) + f(\mathcal{X} \cap \mathcal{Y}),
    \end{equation*}
    for any $\mathcal{X}, \mathcal{Y} \subseteq \Omega$.
    
\end{definition}
This property captures the intuitive notion of diminishing returns: adding an element to a smaller set yields a larger marginal gain than adding it to a larger set. 
Submodular functions play a crucial role in combinatorial optimization, as convex and concave functions do in continuous optimization. While maximizing or minimizing a general set-valued function is typically NP-hard, submodular functions admit efficient approximation algorithms with provable theoretical guarantees, which we will discuss in later sections.

\begin{definition}
\label{def_monotone}
    A set-valued function $f$ is monotone if for all $\mathcal{X} \subseteq \mathcal{Y} \subseteq \Omega$, we have $f(\mathcal{X}) \leq f(\mathcal{Y})$. 
\end{definition}
Monotonicity further strengthens the behaviors of submodular functions in combinatorial optimization, as many greedy algorithms achieve better approximation guarantees when the objective is both submodular and monotone. We will discuss this in detail in the following sections.

\subsection{Low Rank Adaptation}
In transfer learning, the model parameters $\bW_{\text{ft}}$ are fine-tuned based on pre-trained weights $\bW_{\text{pt}}$ as follows:
\begin{equation*}
    \bW_{\text{ft}} = \bW_{\text{pt}} + \Delta \bW,
\end{equation*}
where $\bW_{\text{pt}}, \bW_{\text{ft}}, \Delta \bW \in \mathbb{R}^{n_2 \times n_1}$.
Standard fine-tuning updates the entire parameter matrix  $\bW_{\text{pt}}$  by $\Delta\bW$, providing full flexibility but at the cost of significant computational and memory overhead.
LoRA~\cite{hu2022lora} addresses this inefficiency by imposing a low-rank structure on the update matrix $\Delta \bW$, reducing both parameter count and computational cost. 
Specifically, LoRA parameterizes the update $\Delta \bW$ by the Burer–Monteiro factorization:
\begin{equation*}
    \Delta \bW = \bB \bA,
\end{equation*}
where $\bA \in \mathbb{R}^{r \times n_1}$ and $\bB \in \mathbb{R}^{n_2 \times r}$, and $r \ll \min\{n_1, n_2\}$.
Instead of optimizing the full matrix $\Delta \bW$, LoRA focuses on training the smaller matrices $\bA$ and $\bB$, thereby reducing the number of parameters and computational complexity.

In addition to reducing the number of trainable parameters from $n_1 n_2$ to $r(n_1 + n_2)$, this low-rank constraint introduces useful inductive bias. It helps filter out noise, improves generalization, and enhances the robustness of fine-tuning, especially when limited data or computation is available~\cite{malladi2023kernel,zeng2024the,shuttleworth2024lora}.

\subsection{Rank Determination of LoRA}
For LoRA-based fine-tuning of a multi-layer neural network $\phi\left(\cdot;\Theta\right)$, the model parameters are denoted as $\Theta := \{\bW_{1}, \bW_{2}, \cdots, \bW_{L}\}$, where each layer $\ell$ follows the update rule:
\begin{equation*}
    \bW_{\text{ft},\ell} = \bW_{\text{pt},\ell} + \bB_{\ell} \bA_{\ell},
\end{equation*}
where $\bW_{\text{ft},\ell}\in \mathbb{R}^{n_{\ell+1} \times n_{\ell}}$ denotes the fine-tuned parameters, $ \bW_{\text{pt},\ell} \in \mathbb{R}^{n_{\ell+1} \times n_{\ell}}$ the pre-trained parameters, $\bA_{\ell} \in \mathbb{R}^{r_{\ell} \times n_{\ell}}$, and $\bB_{\ell} \in \mathbb{R}^{n_{\ell+1} \times r_{\ell}}$ the LoRA components.
Here, $n_{\ell}$ represents the width of the $\ell$-th layer, and $r_{\ell}$ is the assigned LoRA rank. Given a global rank budget $b$ (i.e., $\sum_{\ell=1}^{L} r_{\ell} \leq b$)), a key challenge is deciding how to allocate the individual ranks $r_{\ell}$ across layers. Most existing LoRA implementations consider these ranks as {\em fixed} hyperparameters, which can result in suboptimal performance if the chosen ranks do not align well with layer and component importance. In contrast, automatic rank determination seeks to address this issue by {\em dynamically} assigning ranks based on the significance of each layer and component, potentially leading to more efficient and effective model compression.


An alternative to the Burer–Monteiro factorization for rank determination in LoRA is to consider the singular value decomposition (SVD) of the parameter updates. Specifically, the update $\Delta \bW$ is factorized by
\begin{equation}
\label{eq_lora_svd}
    \bW_{\text{ft},\ell} = \bW_{\text{pt},\ell} + \bU_{\ell} \bSig_{\ell} \bV_{\ell},
\end{equation}
where $\bU_{\ell} \in \mathbb{R}^{n_{\ell+1} \times r_{\ell}}$, $\bV_{\ell} \in \mathbb{R}^{r_{\ell} \times n_{\ell}}$, and $\bSig_{\ell} := \text{diag}(\bsig_{\ell}) \in \mathbb{R}^{r_{\ell} \times r_{\ell}}$ is a diagonal matrix with $\bsig_{\ell} \in \mathbb{R}^{r_{\ell}}$.
In this formulation, rank determination reduces to identifying which entries of $\bsig_{\ell}$ are effectively zero, as they correspond to components that do not contribute to the rank. 
Yang et al.~\cite{yang2024bayesian} model the singular values $\bsig_{\ell}$ as random variables following a Gaussian distribution, and estimate their distribution using a Bayesian inference framework. Based on the estimated posterior, rank pruning is performed by hypothesis testing, determining which entries of $\bsig_{\ell}$ can be considered insignificant (i.e., close to zero) and thus pruned.

Another line of research regards $\bsig_{\ell}$ as a deterministic parameter and performs pruning based on the dynamics of the loss function after fine-tuning. 
For notational simplicity, we denote the pre-trained model parameters as $\Theta_{\text{pt}}:= \left\{ \bW_{\text{pt},1}, \cdots, \bW_{\text{pt},L}\right\}$, and the fine-tuned model parameters as $\Theta_{\text{ft}}:= \left\{ \bW_{\text{ft},1}, \cdots, \bW_{\text{ft},L}\right\}$ obtained by \Cref{eq_lora_svd}. The corresponding LoRA components are denoted by $\Theta_{\text{LoRA}} := \left\{ \bU_{\ell}, \bV_{\ell}, \bsig_{\ell}\right\}_{\ell \in [L]}$.
For convenience, we define the operator $\oplus$ such that the fine-tuned parameters can be written as $\Theta_{\text{ft}} := \Theta_{\text{pt}} \oplus \Theta_{\text{LoRA}}$ in accordance with the decomposition in \Cref{eq_lora_svd}.
LoRA fine-tuning aims to minimize the following objective function:
\begin{equation}
\label{eq_lora_finetuning}
    \mathcal{L}_{\text{ft}}\left(\Theta_{\text{LoRA}}\right) = \frac{1}{N} \sum_{i=1}^{N} \mathrm{Loss}\left(\phi\left(\bx_{i};\Theta_{\text{pt}} \oplus \Theta_{\text{LoRA}}\right) ; \by_{i} \right),
\end{equation}
where $\mathrm{Loss}(\cdot;\cdot)$ represents the loss function used to compare predictions and ground-truth labels, and $\{\bx_{i}, \by_{i}\}_{i=1}^{N}$ denotes the training dataset.
The trainable parameters in this setting are those within $\Theta_{\text{LoRA}}$.
In practice, the LoRA rank $r_{\ell}$ is typically assigned uniformly across all layers and chosen to be larger than necessary when no prior knowledge is available.
Under this setup, the rank determination problem seeks to prune the diagonal entries of $\bSig_{\ell}$ (equivalently, the entries of $\bsig_{\ell}$), such that the total number of nonzero singular values across all layers remains within a given budget $b$, while maintaining competitive model performance.
The key challenge lies in determining both (i) how to allocate the rank budget across layers and (ii) which components of $\bsig_{\ell}$ should be retained. To formalize this, we define:
\begin{equation}
    \mathcal{L}\left(\{\Sl\}_{\ell \in [L]}\right) := \mathcal{L}_{\text{ft}}\left(\left[\Theta_{\text{LoRA}}\right]_{\{\Sl\}_{\ell \in [L]}}\right),
\end{equation}
where $\mathcal{S}_{\ell} \subseteq [r_{\ell}]$ is the set of indices of the retained entries at the $\ell$-th layer, and $\left[\Theta_{\text{LoRA}}\right]_{\{\mathcal{S}_{\ell}\}_{\ell \in [L]}} := \left\{ (\bU_{\ell})_{\mathcal{S}_{\ell}}, (\bV_{\ell})_{\mathcal{S}_{\ell}}, [\bsig_{\ell}]_{\mathcal{S}_{\ell}} \right\}_{\ell \in [L]}$. The goal is to solve the following combinatorial optimization problem with cardinality constraints:
\begin{equation}
    \label{eq_prob_nonlinear}
    \begin{split}
        \min_{\{\mathcal{S}_{\ell}\}_{\ell \in [L]}} & \mathcal{L}\left(\{\Sl\}_{\ell \in [L]}\right),\\
        \text{s.t.} & \quad \sum_{\ell=1}^{L}|\Sl| \leq b,
    \end{split}
\end{equation}
where $b$ denotes the total rank budget for the whole model.
Discarding an entry of $\bsig_{\ell}$ is equivalent to setting its corresponding singular value to zero. Therefore, the optimization seeks to prune singular values that contribute the least to performance, such that the loss increase is minimized. Ideally, the pruned model performs comparably to or in some cases better than the original model, while significantly reducing the parameter count.
However, \Cref{eq_prob_nonlinear} is a combinatorial optimization problem over $\sum_{\ell=1}^{L} r_{\ell}$ binary decision variables. Solving such a problem exactly is NP-hard and computationally intractable in general.

For analytical insight, we decompose the pruning procedure as:
\begin{equation}
\label{eq_prunening}
    \mathcal{L}\left(\{\Sl\}_{\ell \in [L]}\right) - \mathcal{L}_{\text{ft}}\left(\Theta_{\text{LoRA}}\right) + \mathcal{L}_{\text{ft}}\left(\Theta_{\text{LoRA}}\right).
\end{equation}
where the first two terms reflect the change in loss due to pruning, and the third term is the loss of the unpruned fine-tuned model.
Note that $\mathcal{L}_{\text{ft}}\left(\Theta_{\text{LoRA}}\right)$ is a fixed quantity after the fine-tuning stage.
Previous work has adopted a first-order (linear) relaxation of the loss difference $\mathcal{L}\left(\{\Sl\}_{\ell \in [L]}\right) - \mathcal{L}_{\text{ft}}\left(\Theta_{\text{LoRA}}\right)$, resulting in the following approximation:
\begin{equation*}
    \begin{split}
        & \mathcal{L}\left(\{\Sl\}_{\ell \in [L]}\right) - \mathcal{L}_{\text{ft}}\left(\Theta_{\text{LoRA}}\right)\\
        & = \mathcal{L}_{\text{ft}}\left(\left[\Theta_{\text{LoRA}}\right]_{\{\Sl\}_{\ell \in [L]}}\right) - \mathcal{L}_{\text{ft}}\left(\Theta_{\text{LoRA}}\right)\\
        & \approx \sum_{\ell=1}^{L} \left\langle \nabla_{\bsig_{\ell}}\mathcal{L}_{\text{ft}}\left(\Theta_{\text{LoRA}}\right), \left[\bsig_{\ell}\right]_{\Sl} - \bsig_{\ell}\right\rangle \\
        & = \sum_{\ell=1}^{L}\left\langle \left[\nabla_{\bsig_{\ell}}\mathcal{L}_{\text{ft}}\left(\Theta_{\text{LoRA}}\right)\right]_{\Sl^{-}}, -\left[\bsig_{\ell}\right]_{\Sl^{-}} \right\rangle,
    \end{split}
\end{equation*}
where $\Sl^{-} := [r_{\ell}] \setminus \Sl$ denotes the complement of the selected indices. This leads to the following surrogate cardinality-constrained combinatorial optimization problem:
\begin{equation}
\label{eq_linear_relaxation}
    \begin{split}
        \min_{\{\mathcal{S}_{\ell}\}_{\ell \in [L]}} & \sum_{\ell=1}^{L} \left\langle \left[\nabla_{\bsig_{\ell}}\mathcal{L}_{\text{ft}}\left(\Theta_{\text{LoRA}}\right)\right]_{\Sl^{-}}, -\left[\bsig_{\ell}\right]_{\Sl^{-}} \right\rangle,\\
    \text{s.t.} & \quad \sum_{\ell=1}^{L}|\Sl| \leq b.
    \end{split}
\end{equation}
Zhang et al.~\cite{zhang2023adaptive} proposed AdaLoRA, which prunes singular values based on their estimated sensitivities. Their method builds upon the linear relaxation in \Cref{eq_linear_relaxation}, with the additional use of elementwise absolute values in the objective to prioritize components with large magnitude impact.

\section{The Proposed Method}
\label{sec_method}
In this section, we first highlight the limitations of the previous rank determination method based on linear relaxation. Building on this observation and motivated by the smoothness properties of the loss function, we propose using a quadratic approximation of the objective, which provides a more accurate capture of the loss landscape. Based on this formulation, we reformulate the rank determination problem as a submodular function maximization problem and adopt the greedy algorithm to solve it efficiently, with theoretical approximation guarantees.
We further extend this method by developing an alternating algorithm that jointly performs rank determination and parameter updates. This integrated strategy aims to further improve the effectiveness of LoRA fine-tuning.

\subsection{Motivation}
The first-order expansion of the loss function used in \Cref{eq_linear_relaxation} may not be a reliable or effective approximation for rank determination. Here, we elaborate on its limitations in detail.
The rank determination typically involves two stages. In the first stage, we fine-tune the model using LoRA with overestimated ranks based on limited prior knowledge. In the second stage, we prune some less important singular values to fit within a total rank budget. Our focus is specifically on the second stage that selects singular values whose removal has less significant impact on performance.
Suppose that the LoRA fine-tuning has been sufficiently optimized and the resulting parameters have reached a stationary point $\Theta_{\text{LoRA}}$ of the objective in \Cref{eq_lora_finetuning}. In this case, we have
\begin{equation*}
    \nabla_{\bsig_{\ell}}\mathcal{L}_{\text{ft}}\left(\Theta_{\text{LoRA}}\right) = \bm{0},
\end{equation*}
which implies that the first-order expansion in \Cref{eq_linear_relaxation} evaluates to zero. Consequently, the linearized objective provides no meaningful signal for pruning, making the optimization problem ill-posed. Notably, this failure is not due to the inadequacy of the expansion, but rather to its strong dependence on gradient information that vanishes at the stationary point.
In practice, LoRA fine-tuning is often carried out to convergence, resulting in nearly stationary solutions. This makes rank pruning based on first-order expansion unreliable and unstable.

To further illustrate this issue, let us consider a simple toy example:
$\mathcal{L}_{\text{ft}}(\sigma_1, \sigma_2) = (\sigma_{1} - \sigma_{2} + 1)^2$. Also consider the  point $(\sigma_1, \sigma_2) = (1, 2.1)$ which is close to the stationary point $(1,2)$. 
In this case, it would be more appropriate to prune $\sigma_{1}$, as $\mathcal{L}_{\text{ft}}(0,2.1) = 1.1^2$ and $\mathcal{L}_{\text{ft}}(1,0) = 4$. However, the gradient at $(\sigma_1, \sigma_2) = (1, 2.1)$  is $\nabla_{\bsig}\mathcal{L}(\sigma_1,\sigma_2) = (-0.2, 0.2)$, which would mislead the linear relaxation in \Cref{eq_linear_relaxation} into pruning $\sigma_2$.
Now consider a slight perturbation of the original point $(\sigma_1, \sigma_2) = (1, 2.1)$  to $(\sigma_1, \sigma_2) = (1.1, 2)$, where the gradient becomes positive for $\sigma_1$, and the linear relaxation would correctly prune $\sigma_1$. This demonstrates that small variations near stationary points can drastically alter the decision of which singular value to prune, despite the fact that the underlying model behavior is essentially unchanged.
This sensitivity arises because the linearization overemphasizes singular values that deviate more from optimality, rather than those that contribute most meaningfully to the loss. 
As a result, the behavior of \Cref{eq_linear_relaxation} is tightly coupled with the optimality of the fine-tuned parameters $\Theta_{\text{LoRA}}$, making it an overly crude, and hence unreliable, basis for rank pruning.
This example underscores the broader insight that first-order information is inadequate for approximating the nonlinear loss surface near stationary points. To address this limitation, we propose using a second-order approximation that captures curvature and provides a more stable and accurate modeling for singular value pruning, especially in the nearly converged regime typical of LoRA fine-tuning.

\subsection{Hessian-Guided Rank Determination}
To address the limitations of the first-order approximation discussed earlier, we propose adopting a second-order (quadratic) expansion of the loss function. This approach provides a more accurate approximation of the objective and avoids the degeneracy that arises when the fine-tuning step reaches near-stationarity.
Let $\bsig \in \mathbb{R}^{\sum_{\ell=1}^{L} r_{\ell}}$ denote the concatenation of all singular value vectors $\bsig_{\ell}$ across layers $\ell \in [L]$. We define $\mathcal{S} \subseteq \left[\sum_{\ell=1}^{L} r_{\ell}\right]$ as the set of indices corresponding to the singular values $\bsig$ of the whole model we choose to retain.
Recall that the rank-pruned objective difference can be approximated using a second-order Taylor expansion:
\begin{equation*}
    \begin{split}
        & \mathcal{L}\left(\{\Sl\}_{\ell \in [L]}\right) - \mathcal{L}_{\text{ft}}\left(\Theta_{\text{LoRA}}\right)\\
        & = \mathcal{L}_{\text{ft}}\left(\left[\Theta_{\text{LoRA}}\right]_{\{\Sl\}_{\ell \in [L]}}\right) - \mathcal{L}_{\text{ft}}\left(\Theta_{\text{LoRA}}\right)\\
        & \approx \left\langle \nabla_{\bsig}\mathcal{L}_{\text{ft}}\left(\Theta_{\text{LoRA}}\right), \left[\bsig\right]_{\mathcal{S}} - \bsig\right\rangle + \frac{1}{2} \left(\left[\bsig\right]_{\mathcal{S}} - \bsig\right)^{\top} \nabla_{\bsig}^2 \mathcal{L}_{\text{ft}}\left(\Theta_{\text{LoRA}}\right) \left(\left[\bsig\right]_{\mathcal{S}} - \bsig\right)  \\
        & = \left\langle \left[\nabla_{\bsig}\mathcal{L}_{\text{ft}}\left(\Theta_{\text{LoRA}}\right)\right]_{\mathcal{S}^{-}}, -\left[\bsig\right]_{\mathcal{S}^{-}} \right\rangle + \frac{1}{2} \left[\bsig\right]_{\mathcal{S}^{-}}^{\top} \left[\nabla_{\bsig}^2 \mathcal{L}_{\text{ft}}\left(\Theta_{\text{LoRA}}\right)\right]_{\mathcal{S}^{-}} \left[\bsig\right]_{\mathcal{S}^{-}},
    \end{split}
\end{equation*}
where $\nabla_{\bsig}^2 \mathcal{L}_{\text{ft}}\left(\Theta_{\text{LoRA}}\right)$ is the Hessian matrix of the fine-tuning objective function with respect to singular values $\bsig$ and $\mathcal{S}^{-}:= \left[\sum_{\ell=1}^{L} r_{\ell}\right] \setminus \mathcal{S}$ represents the complement of $\mathcal{S}$.
Then, we reformulate the approximation of \Cref{eq_prob_nonlinear} into the following combinatorial problem with quadratic objective:
\begin{equation}
\label{eq_quadratic}
   \begin{split}
        \min_{\mathcal{S} \subseteq \left[\sum_{\ell=1}^{L} r_{\ell}\right]} & \left\langle \left[\nabla_{\bsig}\mathcal{L}_{\text{ft}}\left(\Theta_{\text{LoRA}}\right)\right]_{\mathcal{S}^{-}}, -\left[\bsig\right]_{\mathcal{S}^{-}} \right\rangle + \frac{1}{2} \left[\bsig\right]_{\mathcal{S}^{-}}^{\top} \left[\nabla_{\bsig}^2 \mathcal{L}_{\text{ft}}\left(\Theta_{\text{LoRA}}\right)\right]_{\mathcal{S}^{-}} \left[\bsig\right]_{\mathcal{S}^{-}}, \\
        \text{s.t.} & \quad |\mathcal{S}| \leq b.
   \end{split}
\end{equation}
This formulation incorporates both first-order sensitivity and second-order curvature, enabling more robust rank determination, especially in regions near stationary points, where gradients vanish and linear approximations become unreliable. 
Moreover, the inclusion of curvature information allows the model to better capture the structure of the underlying loss landscape, leading to more accurate rank determination.
Beyond interpreting \Cref{eq_quadratic} as a second-order Taylor expansion of the nonlinear objective function, it can also be understood from a geometric perspective. When the fine-tuned parameters are at or near a local minimum, the gradient vanishes $\nabla_{\bsig}\mathcal{L}_{\text{ft}}=\bm{0}$, and the Hessian matrix $\nabla^2_{\bsig}\mathcal{L}_{\text{ft}}$ is positive semi-definite. Under this condition, \Cref{eq_quadratic} simplifies to the problem of minimizing the quadratic form $\|\bsig\|_{\nabla^2_{\bsig}\mathcal{L}_{\text{ft}}}^2 := \bsig^{\top} \nabla^2_{\bsig}\mathcal{L}_{\text{ft}} \bsig$, which represents the magnitude of $\bsig$ measured in the norm induced by the Hessian. This norm quantifies the curvature-aware sensitivity of each entry in $\bsig$, therefore, providing a reliable criterion of singular value pruning for rank determination based on second-order information.

Recall that the motivation behind considering \Cref{eq_linear_relaxation} is to relax the intractable optimization problem in \Cref{eq_prob_nonlinear} into a solvable form. However, the quadratic formulation in \Cref{eq_quadratic}, while more intuitively convincing and theoretically sound, remains NP-hard due to its combinatorial and non-separable structure. This motivates us to further simplify \Cref{eq_quadratic} into a tractable approximation.
A natural relaxation is to consider only the {\em diagonal} elements of the Hessian matrix, rather than the full second-order structure. This diagonal relaxation ensures that the objective becomes separable across the variables, greatly reducing computational complexity. In this case, the optimization problem reduces to selecting the $b$ elements of $\bsig$ that contribute the most to the quadratic objective. Specifically, we approximate \Cref{eq_quadratic} as follows:
\begin{equation}
\label{eq_quadratic_diagonal}
    \begin{split}
        \min_{\mathcal{S} \subseteq \left[\sum_{\ell=1}^{L} r_{\ell}\right]} & \left\langle \left[\nabla_{\bsig}\mathcal{L}_{\text{ft}}\left(\Theta_{\text{LoRA}}\right)\right]_{\mathcal{S}^{-}}, -\left[\bsig\right]_{\mathcal{S}^{-}} \right\rangle + \frac{1}{2} \left[\bsig\right]_{\mathcal{S}^{-}}^{\top} \left[\bD\right]_{\mathcal{S}^{-}} \left[\bsig\right]_{\mathcal{S}^{-}}, \\
        \text{s.t.} & \quad |\mathcal{S}| \leq b,
   \end{split}
\end{equation}
where $\bD := \text{diag}\left(\nabla_{\bsig}^2 \mathcal{L}_{\text{ft}}\left(\Theta_{\text{LoRA}}\right)\right)$ denotes the diagonal of the Hessian.
In this setting, the optimal solution corresponds to selecting the $b$ indices with the largest loss sensitivity and curvature penalty of the corresponding singular value, i.e., 
\begin{equation*}
    - \left[\nabla_{\bsig}\mathcal{L}_{\text{ft}}\left(\Theta_{\text{LoRA}}\right)\right]_{i} \cdot \sigma_{i} + \frac{1}{2} D_{i,i} \cdot \sigma_{i}^2.
\end{equation*}

While the diagonal relaxation makes the problem in \Cref{eq_quadratic} tractable, it comes at the cost of discarding much of the Hessian information. This simplification may lead to suboptimal decisions and performance degradation for rank determination. It raises a natural question: \textit{can we strike a better balance between retaining second-order information and ensuring computational tractability?}
At one extreme, using the full Hessian results in a more accurate modeling of the loss landscape but leads to an NP-hard combinatorial optimization problem. At the other extreme, relaxing the Hessian to its diagonal form makes the objective fully separable and easy to solve, but ignores crucial interactions between variables.
We address this challenge by proposing an intermediate relaxation, one that preserves more of the Hessian structure than the diagonal approximation, while still allowing for efficient optimization. 
To achieve this, we reformulate the objective as a submodular function maximization problem, enabling the use of greedy algorithms with provable theoretical approximation guarantees.

\subsection{Submodular Function Maximization}
Motivated by the well-developed theoretical properties of submodular functions, we propose to adopt the greedy algorithm as an efficient solver for approximately solving \Cref{eq_quadratic}. 
Submodular function maximization under cardinality constraints admits strong approximation guarantees, making greedy algorithms particularly appealing in our setting.
The following theorem summarizes the classical theoretical guarantees for greedy algorithms applied to cardinality-constrained submodular maximization. These results justify the use of greedy and randomized greedy algorithms as practical solvers for the proposed rank determination objective.
\begin{theorem}[See~\cite{krause2014submodular,fujishige2005submodular}]
\label{theorem_greedy_algorithm}
    Let  $f: 2^{[r]} \to \mathbb{R}$  be a non-negative objective function. 
 Consider the following cardinality-constrained combinatorial optimization problem:
    \begin{equation*}
        \begin{split}
            \max_{\mathcal{S} \subseteq [r]} &\quad f(\mathcal{S}),\\
            \text{s.t.} & \quad |\mathcal{S}| \leq b,
        \end{split}
    \end{equation*}
    If $f$ is submodular and monotone, then the solution $\mathcal{S}^{\#}$ obtained from the greedy algorithm (\Cref{alg_greedy}) satisfies:
    \begin{equation*}
        f(\mathcal{S}^{\#}) \geq (1-1/e) \cdot f(\mathcal{S}^{*}),
    \end{equation*}
    where $\mathcal{S}^{*}$ denotes an optimal solution.
    Moreover, if the monotonicity condition does not hold, then the randomized greedy algorithm (\Cref{alg_randomized_greedy}) returns a solution $\mathcal{S}^{\#}$ satisfying:
    \begin{equation*}
        \mathbb{E}\left[f(\mathcal{S}^{\#})\right] \geq  (1/e) \cdot  f(\mathcal{S}^{*}).
    \end{equation*}
\end{theorem}

The above results imply that the greedy algorithm, which is computationally practical, achieves a $(1-1/e)$-approximation ratio to the optimal solution when the objective function is submodular and monotone. In the absence of monotonicity, the randomized greedy algorithm can be adopted to avoid getting trapped in poor local optima, leading to an expected approximation ratio of $1/e$. It is important to note that, despite these guarantees, finding an {\em exact} maximizer of a submodular function is NP hard. However, the submodularity property imposes useful structure and regularity on the objective function, ensuring that the solutions obtained from greedy algorithms are provably close to optimal in terms of objective value. 
In contrast, for general (non-submodular) functions, greedy methods can perform arbitrarily poorly unless additional assumptions are imposed. Therefore, in the context of submodular function maximization, greedy algorithms provide a theoretically grounded and efficient approach.

\begin{algorithm}[t!]
\caption{Greedy Algorithm}
\label{alg_greedy}
\begin{algorithmic}[1]
\State Initialize \( \mathcal{S} \gets \emptyset \)
\For{\( i = 1 \) to \( b \)}
    \State Select \( e \in [r] \setminus \mathcal{S} \) that maximizes the marginal gain:
    \[
    e = \arg\max_{j \in [r] \setminus \mathcal{S}} f(\mathcal{S} \cup \{j\}) - f(\mathcal{S})
    \]
    \State Update the solution: \( \mathcal{S} \gets \mathcal{S} \cup \{e\} \)
\EndFor
\State \Return \( \mathcal{S} \)
\end{algorithmic}
\end{algorithm}

\begin{algorithm}[t!]
\caption{Randomized Greedy Algorithm}
\label{alg_randomized_greedy}
\begin{algorithmic}[1]
\State Initialize \( \mathcal{S} \gets \emptyset \)
\For{\( i = 1 \) to \( b \)}
    \State Select \( e \in [r] \setminus \mathcal{S} \) with probability proportional to:
    \[
    \max\{ f(\mathcal{S} \cup \{j\}) - f(\mathcal{S}), 0 \}, \quad \text{for } j \in [r] \setminus \mathcal{S}
    \]
    \State Update the solution: \( \mathcal{S} \gets \mathcal{S} \cup \{e\} \)
\EndFor
\State \Return \( \mathcal{S} \)
\end{algorithmic}
\end{algorithm}

To make use of the theoretical guarantees provided in \Cref{theorem_greedy_algorithm}, we must ensure that the objective function in \Cref{eq_quadratic} is submodular. 
Here, we aim to reformulate \Cref{eq_quadratic} and modify the Hessian matrix such that the resulting objective function becomes submodular.
The following theorem provides a key condition under which a set-valued quadratic function is submodular, thus enabling the use of greedy algorithms with approximation guarantees.

\begin{theorem}
\label{theorem_submodular_Hessian}
Let $\bc \in \mathbb{R}^{r}$, $\bx \in \mathbb{R}^{r}$, and $\bH \in \mathbb{S}^{r}$ be given vectors and a symmetric matrix, respectively. Consider the set-valued function:
    \begin{equation*}
        f(\mathcal{S}) = \left[\bc\right]_{\mathcal{S}^{-}}^{\top}\left[\bx\right]_{\mathcal{S}^{-}} + \frac{1}{2} \left[\bx\right]_{\mathcal{S}^{-}}^{\top} \left[\bH\right]_{\mathcal{S}^{-}} \left[\bx\right]_{\mathcal{S}^{-}},
    \end{equation*}
    where $\mathcal{S} \subseteq [r]$ and $\mathcal{S}^{-} = [r] \setminus \mathcal{S}$. Then $f$
    is submodular if and only if
    \begin{equation*}
        H_{i,j} x_i x_j \leq 0,
    \end{equation*}
    for any $i, j \in [r]$ and $i \neq j$.
\end{theorem}

\begin{proof}
    We prove the result using the definition of submodularity from \Cref{def_submodular}. The submodularity condition requires that:
    \begin{equation*}
    f(\mathcal{X}) + f(\mathcal{Y}) \geq f(\mathcal{X} \cup \mathcal{Y}) + f(\mathcal{X} \cap \mathcal{Y}).
    \end{equation*}
    Let $\mathcal{X} = [r] \setminus \{i\}$ and $\mathcal{Y} = [r] \setminus \{j\}$, for any $i, j \in [r]$ and $i \neq j$, we have $\mathcal{X} \cap \mathcal{Y} = [r]$ and $\mathcal{X} \cup \mathcal{Y} = [r] \setminus \{i,j\}$. Then, substituting them into the above inequality, we obtain 
    \begin{equation*}
        f\left([r] \setminus \{i\}\right) + f\left([r] \setminus \{j\}\right) \geq f\left([r]\right) + f\left([r] \setminus \{i, j\}\right).
    \end{equation*}
    Noting that all terms cancel except the interaction terms between indices $i$ and $j$, we have
    \begin{equation*}
        H_{i,i} x_i^2 + H_{j,j} x_j^2 \geq H_{i,i} x_i^2 + H_{j,j} x_j^2 + 2 H_{i,j} x_i x_j,
    \end{equation*}
    which simplifies to:
    \begin{equation*}
        H_{i,j} x_i x_j \leq 0.
    \end{equation*}

    Conversely, if $H_{i,j} x_i x_j \leq 0$ holds for any $i, j \in [r]$ and $i \neq j$, then it is straightforward to verify that the submodularity condition $f(\mathcal{X}) + f(\mathcal{Y}) \geq f(\mathcal{X} \cup \mathcal{Y}) + f(\mathcal{X} \cap \mathcal{Y})$ holds, for all $\mathcal{X}, \mathcal{Y} \subseteq \Omega$, since
    $\mathcal{X} \cup \mathcal{Y} \setminus \mathcal{X} \cap \mathcal{Y} = \left(\mathcal{X} \setminus \mathcal{Y} \right) \cup \left(\mathcal{Y} \setminus \mathcal{X}\right)$. This completes the proof.
\end{proof}

To clearly present our method, we begin by rewriting \Cref{eq_quadratic} as the following maximization problem:
\begin{equation}
\label{eq_quadratic_max}
   \begin{split}
        \max_{\mathcal{S} \subseteq \left[\sum_{\ell=1}^{L} r_{\ell}\right]} & \left\langle \left[\nabla_{\bsig}\mathcal{L}_{\text{ft}}\left(\Theta_{\text{LoRA}}\right)\right]_{\mathcal{S}^{-}},\left[\bsig\right]_{\mathcal{S}^{-}} \right\rangle - \frac{1}{2} \left[\bsig\right]_{\mathcal{S}^{-}}^{\top} \left[\nabla_{\bsig}^2 \mathcal{L}_{\text{ft}}\left(\Theta_{\text{LoRA}}\right)\right]_{\mathcal{S}^{-}} \left[\bsig\right]_{\mathcal{S}^{-}}, \\
        \text{s.t.} & \quad |\mathcal{S}| \leq b,
   \end{split}
\end{equation}
According to \Cref{theorem_submodular_Hessian}, the objective function is submodular if $\left[\nabla_{\bsig}^2 \mathcal{L}_{\text{ft}}\left(\Theta_{\text{LoRA}}\right)\right]_{i,j} \cdot \sigma_{i} \sigma_{j} \geq 0$, for all $i \neq j$. 
In particular, if we directly relax the Hessian to its diagonal, then the condition is trivially satisfied since $D_{i,j} = 0$, for $i, j \in \left[\sum_{\ell=1}^{L} r_{\ell}\right]$ and $i \neq j$. 
Thus, the diagonal relaxation used in \Cref{eq_quadratic_diagonal} leads to a submodular objective. 
In fact, since all variables in $\bsig$ are completely decoupled in the diagonal case, the objective becomes strictly separable, which is a stronger condition than submodularity. Consequently, the greedy algorithm finds the global maximizer exactly in this case.

However, relying solely on the diagonal elements of the Hessian sacrifices valuable second-order information. To balance between  tractability of solving \Cref{eq_quadratic_max} and the fidelity of Hessian information, we propose modifying the Hessian to preserve as much curvature information as possible while still ensuring submodularity. Since submodularity is determined solely by the quadratic term, we focus on projecting the Hessian to a nearby matrix that satisfies the submodularity condition. Specifically, we solve the following projection problem:
\begin{equation}
\label{eq_projection_submodular}
    \begin{split}
        \min_{\bG \in \mathbb{S}^{\sum_{\ell=1}^{L} r_{\ell}}} & \left\| \bG - \nabla_{\bsig}^2 \mathcal{L}_{\text{ft}}\left(\Theta_{\text{LoRA}}\right)\right\|_F^2,\\
        \text{s.t.} & \quad G_{i,j} \cdot \sigma_{i} \sigma_{j} \geq 0,~\text{for}~i, j \in \left[\sum_{\ell=1}^{L} r_{\ell}\right] ~\text{and}~i \neq j .
    \end{split}
\end{equation}
This projection preserves as much of the original Hessian as possible under the Frobenius norm, while enforcing the structural condition required for submodularity.
Since all $G_{i,j}$ terms are separable in both the objective and the constraint of \Cref{eq_projection_submodular}, we can derive a closed-form solution for the projection:
\begin{equation}
\label{eq_proj_hessian}
    G_{i,j} = \left\{ \begin{array}{ll}
        \left[\nabla_{\bsig}^2 \mathcal{L}_{\text{ft}}\left(\Theta_{\text{LoRA}}\right)\right]_{i,j}, & \text{if}~ \left[\nabla_{\bsig}^2 \mathcal{L}_{\text{ft}}\left(\Theta_{\text{LoRA}}\right)\right]_{i,j} \cdot \sigma_{i} \sigma_{j} \geq 0~\text{and}~i \neq j,\\
        0, & \text{if}~ \left[\nabla_{\bsig}^2 \mathcal{L}_{\text{ft}}\left(\Theta_{\text{LoRA}}\right)\right]_{i,j} \cdot \sigma_{i} \sigma_{j} < 0~\text{and}~i \neq j,\\
        \left[\nabla_{\bsig}^2 \mathcal{L}_{\text{ft}}\left(\Theta_{\text{LoRA}}\right)\right]_{i,i}, &\text{if}~i=j.
    \end{array} \right.
\end{equation}
This projection is computationally efficient due to its closed-form nature.
To convert \Cref{eq_quadratic_max} into a submodular function maximization problem, we apply this projection to the Hessian matrix, resulting in a modified matrix denoted by $\bG$.
We then solve the following cardinality-constrained submodular function maximization problem:
\begin{equation}
\label{eq_quadratic_max_submodular}
   \begin{split}
        \max_{\mathcal{S} \subseteq \left[\sum_{\ell=1}^{L} r_{\ell}\right]} & \left\langle \left[\nabla_{\bsig}\mathcal{L}_{\text{ft}}\left(\Theta_{\text{LoRA}}\right)\right]_{\mathcal{S}^{-}},\left[\bsig\right]_{\mathcal{S}^{-}} \right\rangle - \frac{1}{2} \left[\bsig\right]_{\mathcal{S}^{-}}^{\top} \left[\bG\right]_{\mathcal{S}^{-}} \left[\bsig\right]_{\mathcal{S}^{-}}, \\
        \text{s.t.} & \quad |\mathcal{S}| \leq b,
   \end{split}
\end{equation}
by greedy algorithms. 
Compared to the naive diagonal approximation in \Cref{eq_quadratic_diagonal}, which discards all off-diagonal Hessian information, the projected Hessian $\bG$ preserves additional curvature structure while still ensuring submodularity. This leads to more accurate pruning of singular values and improved rank determination.
Although solving \Cref{eq_quadratic_max_submodular} exactly remains NP-hard, the greedy algorithm provides a tractable approximation with a known $1-1/e$ guarantee. This trade-off reflects our goal of striking a balance between two competing aspects: preserving second-order information from the Hessian and ensuring the problem remains solvable with provable guarantees.

The complete algorithm of the proposed SubLoRA for LoRA rank determination is summarized in \Cref{alg_rank_determination}.
Stage 1 involves standard LoRA fine-tuning, where we apply (stochastic) gradient descent to minimize the loss function in \Cref{eq_lora_finetuning}, providing an approximate solution. This step follows classical LoRA training procedures and is not the primary focus of this work.
Stages 2 and 3 constitute the core contribution of this paper. In Stage 2, we approximate the objective in \Cref{eq_prunening} using a second-order Taylor expansion, as formulated in \Cref{eq_quadratic}. To balance the trade-off between computational tractability and fidelity to the curvature of the loss landscape, we project the Hessian matrix according to the constraint in \Cref{eq_projection_submodular}, using the closed-form solution in \Cref{eq_proj_hessian}. This projection ensures that the resulting quadratic objective in \Cref{eq_quadratic_max_submodular} is submodular.
In Stage 3, we solve the resulting cardinality-constrained submodular function maximization problem using either the greedy or randomized greedy algorithm. The final output is a selected subset of singular values to retain, while the remaining LoRA components are discarded accordingly.

\begin{algorithm}[t!]
\caption{SubLoRA}
\label{alg_rank_determination}
\begin{algorithmic}[1]
\State \textbf{Stage 1 (LoRA fine-tuning):} Obtain LoRA parameters $\Theta_{\text{LoRA}}$ by minimizing \Cref{eq_lora_finetuning}
\State \textbf{Stage 2 (Problem formulation):} Calculate the gradient  vector $\nabla_{\bsig}\mathcal{L}_{\text{ft}}\left(\Theta_{\text{LoRA}}\right)$ and Hessian matrix $\nabla_{\bsig}^2 \mathcal{L}_{\text{ft}}\left(\Theta_{\text{LoRA}}\right)$ with respect to the vector of singular values $\bsig$. 
\State Project the Hessian matrix and obtain $\bG$ as in \Cref{eq_proj_hessian} 
\State Formulate the submodular function maximization problem as in \Cref{eq_quadratic_max_submodular}
\State \textbf{Stage 3 (Solvers):} Apply greedy algorithm (\Cref{alg_greedy}) or randomized greedy algorithm (\Cref{alg_randomized_greedy}) to obtain the set of singular values to be kept
\end{algorithmic}
\end{algorithm}

In submodular function maximization, it is well known that without the monotonicity condition, the greedy algorithm (\Cref{alg_greedy}) may be trapped in suboptimal or ill-conditioned points, potentially leading to poor performance~\cite{krause2014submodular,fujishige2005submodular}. 
To mitigate this, randomized greedy algorithms are often adopted, which also provide theoretical guarantees even in the non-monotone setting, as shown in \Cref{theorem_greedy_algorithm}.
However, our empirical observations show that the standard greedy algorithm consistently outperforms the randomized variant in most cases. Notably, in \Cref{alg_rank_determination}, we do not explicitly enforce any conditions ensuring the monotonicity of the objective function. This prompts us to question whether the objective in \Cref{eq_quadratic_max_submodular} is, in fact, monotone in practice. The following lemma provides a sufficient condition under which the objective function is monotone.

\begin{lemma}
    Suppose the LoRA fine-tuning stage (Stage 1 of \Cref{alg_rank_determination}) minimizes the loss in \Cref{eq_lora_finetuning} to a point satisfying the second-order necessary optimality conditions. Then, the objective function in \Cref{eq_quadratic_max_submodular} is monotone.
\end{lemma}

\begin{proof}
    The second-order necessary conditions for the optimality of \Cref{eq_lora_finetuning} imply that
    \begin{equation*}
        \nabla_{\bsig}\mathcal{L}_{\text{ft}}\left(\Theta_{\text{LoRA}}\right) = \bm{0}, \quad \nabla^2_{\bsig}\mathcal{L}_{\text{ft}}\left(\Theta_{\text{LoRA}}\right) \succeq \bm{0}.
    \end{equation*}
    The positive semidefiniteness of the Hessian ensures that its diagonal entries are non-negative. Therefore, the objective function in \Cref{eq_quadratic_max_submodular} simplifies to:
    \begin{equation*}
   \begin{split}
        \max_{\mathcal{S} \subseteq \left[\sum_{\ell=1}^{L} r_{\ell}\right]} &  - \frac{1}{2} \left[\bsig\right]_{\mathcal{S}^{-}}^{\top} \left[\bG\right]_{\mathcal{S}^{-}} \left[\bsig\right]_{\mathcal{S}^{-}}, \\
        \text{s.t.} & \quad |\mathcal{S}| \leq b,
   \end{split}
\end{equation*}
where the projected matrix $\bG$ satisfies $G_{i,j} \cdot \sigma_i \sigma_j \geq 0$ for all $i, j \in [\sum_{\ell=1}^{L} r_{\ell}]$.
For any set $\mathcal{S} \subseteq [\sum_{\ell=1}^{L} r_{\ell}]$ and $i \in [\sum_{\ell=1}^{L} r_{\ell}]$, without  loss of generality that $i \notin \mathcal{S}$, then
\begin{equation*}
    f(\mathcal{S} \cup \{i\}) - f(\mathcal{S}) = 2 \sum_{i^{\prime} \in [\sum_{\ell=1}^{L} r_{\ell}] \setminus \left(\mathcal{S} \cup \{i\} \right)} G_{i^{\prime},i} \cdot \sigma_{i^{\prime}} \sigma_{i} + G_{i,i} \cdot  \sigma_{i}^2 \geq 0.
\end{equation*}
Therefore, the set-valued function $f(\mathcal{S}) :=  - \frac{1}{2} \left[\bsig\right]_{\mathcal{S}^{-}}^{\top} \left[\bG\right]_{\mathcal{S}^{-}} \left[\bsig\right]_{\mathcal{S}^{-}}$ is monotone according to   \Cref{def_monotone}. 
\end{proof}

Note that the above lemma requires that the LoRA fine-tuning step satisfies second-order necessary optimality conditions.
Recent theoretical results~\cite{ge2015escaping,ge2016matrix} suggest that many widely adopted optimizers, such as Adam and SGD, converge to points satisfying second-order optimality conditions with high probability under mild assumptions.
As a result, the objective in \Cref{eq_quadratic_max_submodular} tends to exhibit approximate monotonicity in practice after a sufficient number of iterations in the fine-tuning stage.
Importantly, the greedy algorithm is known to achieve a higher approximation ratio than its randomized variant when the objective function is both submodular and monotone. This observation provides a theoretical explanation for the empirical finding that the greedy algorithm often outperforms the randomized greedy algorithm in \Cref{alg_rank_determination}. Therefore, when Stage 1 of SubLoRA is optimized to a sufficient degree, it is well-justified to adopt the greedy algorithm in Stage 3. Under these conditions, the objective becomes submodular and (approximately) monotone, ensuring that the greedy selection is both efficient and effective.

\subsection{Alternating Training and Rank Determination}
The previous sections focus on a training-free, post-hoc method for LoRA rank determination, where the LoRA components are fixed and the effective rank is reduced by pruning unimportant singular values. Here, we extend our method to a more integrated setting, where LoRA fine-tuning and rank determination are performed simultaneously. Specifically, we aim to solve the following constrained optimization problem:
\begin{equation}
    \label{eq_lora_rank_both}
    \begin{split}
        \min_{\Theta_{\text{LoRA}}} & \quad \mathcal{L}_{\text{ft}}\left(\Theta_{\text{LoRA}}\right),\\
        \text{s.t.} & \quad \left\|\bsig \right\|_{0} \leq b,
    \end{split}
\end{equation}
where a sparsity constraint is imposed on the singular values $\bsig$ to enforce a rank budget $b$. 
However, solving \Cref{eq_lora_rank_both} directly is NP-hard due to the non-convex $\ell_0$ constraint.
To address this challenge, we propose an alternating training strategy that interleaves LoRA parameter updates with rank determination by singular value pruning. The full procedure is described in \Cref{alg_lora_update_rank_determination}.

The algorithm operates in an outer loop of alternating phases. In the parameter update phase, we fine-tune all LoRA parameters using standard gradient-based optimizers for a fixed number of iterations, without imposing any rank constraints. At this stage, the total rank may exceed the target budget $b$. 
In the subsequent rank determination phase, we fix the current LoRA parameters and apply the proposed rank determination method via submodular function maximization to eliminate unimportant singular values, reducing the effective rank under the constraint $\|\bsig\|_0 \leq b$.
The alternating process is repeated for several rounds. From an optimization perspective, the rank determination step with pruning encourages the model to suppress redundant components, therefore guiding the optimization trajectory toward a low-rank solution. 
In turn, subsequent parameter updates refine the remaining components and further reduce the loss. This synergistic interaction enables the model to approximately solve \Cref{eq_lora_rank_both}, achieving both high accuracy and parsimony in representation.

\begin{algorithm}[t!]
\caption{Alternating SubLoRA}
\label{alg_lora_update_rank_determination}
\begin{algorithmic}[1]
\For{$t=0,1,\cdots,T$}
\State \textbf{Stage 1 (Parameter update):} Run several steps of (stochastic) optimization algorithms for updating LoRA parameters $\Theta_{\text{LoRA}}$.
\State \textbf{Stage 2 (Problem formulation):} Calculate the gradient  vector $\nabla_{\bsig}\mathcal{L}_{\text{ft}}\left(\Theta_{\text{LoRA}}\right)$ and Hessian matrix $\nabla_{\bsig}^2 \mathcal{L}_{\text{ft}}\left(\Theta_{\text{LoRA}}\right)$ with respect to the vector of singular values $\bsig$. 
\State Project the Hessian matrix and obtain $\bG$ as in \Cref{eq_proj_hessian} 
\State Formulate the submodular function maximization problem as in \Cref{eq_quadratic_max_submodular}
\State \textbf{Stage 3 (Solvers):} Apply the greedy algorithm (\Cref{alg_greedy}) or randomized greedy algorithm (\Cref{alg_randomized_greedy}) to obtain the set of singular values to be retained
\EndFor
\end{algorithmic}
\end{algorithm}

\section{Applications}
\label{sec_applications}
In this section, we explore potential real-world applications of the proposed SubLoRA method. Specifically, we focus on applying SubLoRA to LoRA fine-tuning of PINNs for solving PDEs. 
PINNs have recently gained attention for their ability to incorporate physical laws directly into the learning process. Efficient transfer learning in PINNs is highly desirable, especially when adapting models to PDEs with new boundary conditions or physical parameters. Furthermore, the inherent smoothness of PDE solutions and PINN loss landscapes makes them particularly well-suited for using second-order information. This is well consistent with our approach, where curvature information contributes to effective rank determination of LoRA.

\subsection{Physics-Informed Machine Learning}
We begin by illustrating how the proposed SubLoRA method can be applied to LoRA fine-tuning of PINNs for solving PDEs.
We consider solving a class of PDEs with different physical parameters $\blam \in \mathcal{P}$:
\begin{equation*}
    \begin{split}
        & \mathcal{D}[\bu_{\blam};\blam] = g(\bx;\blam),\quad\bx \in \Omega, \\
        & \mathcal{B}[\bu_{\bm{\lambda}};\bm{\lambda}] = h(\bx;\blam),\quad\bx \in \partial \Omega, 
    \end{split}
\end{equation*}
where $\mathcal{D}$ and $\mathcal{B}$ are differential operators defined in the interior domain $\Omega$ and on its boundary $\partial \Omega$, respectively. 
The solution associated with a specific physical parameter $\blam$ is denoted as $\bu_{\blam}$.

In physics-informed machine learning~\cite{karniadakis2021physics,raissi2019physics}, neural networks act as surrogates for PDE solutions. 
Let $\phi\left(\bx;\Theta_{\text{pt}}\right)$ denote a neural network parameterized by $\Theta_{\text{pt}}=\left\{ \bW_{\text{pt},1},\bW_{\text{pt},2},\cdots,\bW_{\text{pt},L}\right\}$. 
The training objective for solving the PDE with physical parameters $\blam$ is given by
\begin{align}
       & \mathcal{L}\left(\Theta_{\text{pt}}\right)   = \frac{\mu}{N} \sum_{i=1}^{N} \left\| \mathcal{D}\left[\phi\left(\bx_{i};\Theta_{\text{pt}}\right);\blam\right] - \by_{i} \right\|_2^2 \nonumber \\*
        & \qquad\qquad + \frac{\mu_{b}}{N_{b}} \sum_{j=1}^{N_b} \left\| \mathcal{B}\left[\phi\left(\hat{\bx}_{j};\Theta_{\text{pt}}\right);\blam\right] - \bv_{j} \right\|_2^2,
\end{align}
where $\left\{ (\bx_{i},\by_{i})\right\}_{i \in [N]}$ are interior domain observations with $\by_{i} = g(\bx_{i};\blam)$, and $\left\{ (\hat{\bx}_{j},\bv_{j})\right\}_{j \in [N_{b}]}$ are boundary observations with $\bv_{j} = h(\hat{\bx}_{j};\blam)$. The hyperparameters $\mu >0$ and $\mu_{b}>0$ control the trade-off between the interior residual and the boundary condition residual.

Our objective is to train neural networks (e.g., MLPs) to efficiently approximate solutions $\{\bu_{\blam}\}_{\blam \in \mathcal{P}}$ for the entire class of PDEs. 
Due to structural similarities in the PDE operators, we assume that the solutions exhibit shared underlying patterns. 
Rather than relearning these shared structures from scratch for each individual PDE instance, we adopt LoRA to encode the common information across tasks while enabling efficient adaptation to task-specific variations. This approach not only reduces redundant computation but also significantly improves storage efficiency after fine-tuning.

Given a PDE characterized by physical parameters $\blam_{0}$, we first pre-train a neural network on this PDE, where the resulting model $\phi(\cdot;\Theta_{\text{pt}})$ approximates the corresponding solution of $\bu_{\blam_{0}}$. 
The pre-trained parameters $\Theta_{\text{pt}}$ thus capture the shared knowledge across the PDE class.
For a new PDE with different physical parameters $\hat{\blam} \in \mathcal{P}$, we fine-tune the model using LoRA to obtain $\phi(\cdot;\Theta_{\text{ft}})$, as formulated in \Cref{eq_lora_svd}, which approximates the new solution $\bu_{\hat{\blam}}$. 
The fine-tuning process is efficient because the shared information has already been included, eliminating the need for redundant learning and thus resulting in faster convergence.

Given practical restrictions on storage and computation, it is necessary to determine a compact, task-specific rank for each LoRA component across layers. 
To this end, we apply the proposed rank determination method outlined in \Cref{alg_rank_determination}, which allows training-free pruning of singular values. This step aims to reduce model complexity while minimizing performance degradation. Furthermore, to enhance the overall performance, we recommend the alternating training strategy (alternating SubLoRA) that combines LoRA fine-tuning with iterative rank determination, as described in \Cref{alg_lora_update_rank_determination}.

\subsection{Smoothness and Hessian Information}
Wang et al.~\cite{wang2025transfer} investigate the LoRA fine-tuning of PINNs, but they do not consider  rank determination.
In their work, the LoRA ranks are treated as fixed hyperparameters rather than optimized quantities. To the best of our knowledge, our work is the first to explore both LoRA fine-tuning and rank determination within the context of PINNs.
In comparison to existing approaches such as AdaLoRA~\cite{zhang2023adaptive}, which rely mainly on first-order information (i.e., gradient-based linear approximations as in \Cref{eq_linear_relaxation}), our method utilizes second-order information through quadratic expansion on the objective. This richer characterization allows us to more accurately capture the local geometry of the loss landscape, as formalized in the nonlinear formulation of \Cref{eq_prob_nonlinear}.
In large-scale applications like language modeling, both pre-training and LoRA fine-tuning are typically conducted with first-order methods (e.g., Adam or SGD). 
This is largely due to the prohibitive computational cost of computing full Hessians in high-dimensional parameter spaces. In such settings, gradients alone often suffice to guide optimization effectively, and Hessian-based methods are rarely used.
By contrast, the situation is different in training PINNs, where the loss landscape is often highly nonconvex and more complex due to the inclusion of differential operators. 
Consequently, second-order optimization methods are more commonly adopted. For example, the original PINNs paper~\cite{raissi2019physics} adopts the L-BFGS optimizer, a quasi-Newton method with curvature information. 
Additionally, natural gradient methods, which approximate the Hessian using the Fisher information matrix, have recently been investigated in the PINN literature~\cite{muller2023achieving,mckay2025nearoptimal}. Theoretical results~\cite{bonfanti2024challenges} further demonstrate that Newton-type methods can significantly improve convergence rates for training PINNs.
These observations motivate the use of Hessian-informed techniques for LoRA rank determination in PINNs. Unlike tasks in natural language processing or computer vision, where Hessian information may provide relatively limited benefit, PINNs tend to gain substantially from second-order information. 
Therefore, this paper primarily focuses on applying and evaluating the proposed SubLoRA, which is a Hessian-based rank determination method, within the PINNs setting, particularly for transfer learning across PDEs.

The exact computation of the Hessian matrix is often a major concern in deep neural networks due to the large number of trainable parameters. However, in the context of SubLoRA, this challenge is significantly mitigated. In contrast to traditional second-order optimization methods such as Newton’s method, which require computing the Hessian with respect to all model parameters, SubLoRA computes the Hessian $\nabla_{\bsig}^2 \mathcal{L}_{\text{ft}}(\Theta_{\text{LoRA}})$ only with respect to the singular values $\bsig$. This dramatically reduces the computational and memory requirements.
For example, consider an MLP with dimension $n$ and $L$ layers. The full Hessian over all parameters would involve approximately $\mathcal{O}\left(L^2n^4\right)$ elements. In contrast, SubLoRA maintains only $\mathcal{O}\left(L\right)$ singular values (assuming a fixed LoRA rank per layer), resulting in a Hessian of size  $\mathcal{O}\left(L^2\right)$. Therefore, both the computation and storage of the Hessian $\nabla_{\bsig}^2 \mathcal{L}_{\text{ft}}(\Theta_{\text{LoRA}})$ are substantially more efficient in SubLoRA compared to general second-order optimization methods.

Although our method is broadly applicable to any setting where the loss function is twice differentiable, we believe that the smoothness and structural properties of PINNs make them especially well suited to benefit from this approach. Applying SubLoRA in other domains, such as language models and vision tasks, is a promising direction for future work.

\section{Experiments}
\label{sec_experiments}
In this section, we present comprehensive experiments to evaluate the effectiveness of the proposed SubLoRA method for LoRA rank determination in the context of fine-tuning PINNs. We first examine the training-free setting, where rank determination is performed in the post-training phase without updating the LoRA parameters, as described in \Cref{alg_rank_determination}.
In addition, to further improve model performance, we implement the alternating optimization strategy, as presented in \Cref{alg_lora_update_rank_determination}, where LoRA fine-tuning and rank determination are performed iteratively to better satisfy the rank budget while maintaining the prediction accuracy.

\subsection{Training-Free Rank Determination}
\label{sec_experiments_training_free}
We evaluate the training-free rank determination introduced in \Cref{alg_rank_determination} for solving a class of PDEs with varying physical parameters, as detailed in \Cref{sec_applications}. 
Three representative types of PDEs are considered: elliptic, parabolic, and hyperbolic equations. 
For a given PDE, we first pre-train the neural network (MLPs) and obtain network parameters $\Theta_{\text{pt}}$. We then fine-tune the network by LoRA (with uniform pre-defined ranks across all layers) on a new PDE with another physical parameter and obtain $\Theta_{\text{LoRA}}$. 
The total LoRA rank is intentionally set much higher than the desired budget.
We then apply the proposed rank determination method (e.g., SubLoRA) to remove less informative components, producing budget-constrained model parameters $\hat{\Theta}_{\text{ft}}$ by reallocating rank across layers. 
To evaluate the quality of the rank determination, we compute the relative error (rel) between the predicted solution $\phi(\bx;\Theta)$ and the ground truth solution $\bu(\bx)$ on a test dataset $\{(\bx_i, \bu(\bx_i))\}_{i=1}^{N_t}$, defined as $\sqrt{\frac{\sum_{i=1}^{N_t} \|\phi(\bx_i;\Theta) - \bu(\bx_i)\|_2^2}{\sum_{i=1}^{N_t} \|\bu(\bx_i)\|_2^2}}$.

In the experiments, we compare four different methods of LoRA rank determination. The first is the linearized objective function method in \Cref{eq_linear_relaxation}, denoted as ``LinearLoRA''. 
The second method, named ``DiagLoRA'', utilizes the diagonal second-order approximation introduced in \Cref{eq_quadratic_diagonal}. 
The third method ``SubLoRA'' corresponds to the proposed approach that formulates the rank selection problem as a submodular maximization using the projected Hessian in \Cref{eq_quadratic_max_submodular}, solved by either greedy or randomized greedy algorithms. 
To evaluate the effectiveness of the Hessian projection step in \Cref{eq_projection_submodular}, we also include a fourth method, termed ``HessLoRA'', which uses the exact (unprojected) Hessian from \Cref{eq_quadratic_max} and applies the greedy algorithm for optimization.
All experiments are conducted using a four-layer MLP architecture with three hidden layers. 
Each hidden layer consists of 1000 neurons and is fine-tuned with LoRA using a fixed pre-defined rank of 50 per layer, resulting in a total initial LoRA rank of 100. We vary the total rank budget in the experiments, and each method determines its own layer-wise rank allocation under the given constraint.

\subsubsection{Elliptic Equations}
\label{sec_elliptic}
We consider a class of elliptic equations with varying physical parameters $\blam \in \mathbb{R}^{2}$:
\begin{align}
        - \nabla \cdot \left(a(\bm{x}) \cdot \nabla u(\bm{x};\blam)\right) + \left\| \nabla u(\bm{x};\blam)\right\|_2^2   = g(\bm{x};\blam), \quad \bm{x} &\in \Omega, \nonumber\\*
        u(\bm{x};\blam) = h(\bm{x};\blam), \quad \bm{x} &\in \partial \Omega,
       \end{align} 
where the domain is defined as $\Omega = \{\bx \in \mathbb{R}^{2}: \left\|\bx \right\|_2 \leq 1\}$, and the coefficient function is given by $a(\bx) = 1 + \frac{1}{2}\left\| \bx\right\|_2^2$. The exact solution $\bu(\bx;\blam)$ with the parameter $\blam$ is defined as $u(\bx;\blam) = \sin \big(\frac{\pi \lambda_1}{2}\big( 1 - \left\|\bm{x} \right\|_2 \big)^{2.5} \big) + \lambda_2 \cdot \sin \big(\frac{\pi}{2}\big( 1 - \left\|\bm{x} \right\|_2 \big) \big)$. 
Here, $\lambda_1$ modulates the frequency and $\lambda_2$ controls the variation level. 
We first pre-train a MLP model on the PDE with $\blam =(1,0)$, and then fine-tune it by standard LoRA on PDEs with $\blam=(1,1)$, $(1,5)$, and $(2,1)$. 
After fine-tuning, we perform rank determination under a fixed total budget using the following methods: the linearized first-order method (LinearLoRA), the diagonal second-order approximation (DiagLoRA), the proposed submodular second-order method with greedy (SubLoRA-G) and randomized greedy (SubLoRA-R) solvers, and the exact second-order method without projection, solved using the greedy algorithm (HessLoRA-G).

Experimental results are presented in \Cref{fig_static_elliptic}. We observe that all second-order methods, including DiagLoRA, SubLoRA, and HessLoRA, consistently and significantly outperform the linearized method LinearLoRA in terms of both training loss and validation error. This supports our key motivation that linear approximation is insufficient for effective rank determination.
The randomized greedy algorithm shows inferior performance compared to its deterministic counterpart, primarily due to its weaker approximation guarantee, as established in \Cref{theorem_greedy_algorithm}. Among second-order methods, SubLoRA outperforms DiagLoRA, highlighting the benefits of incorporating richer Hessian information beyond diagonal elements. 
This validates the necessity of designing the submodular framework in this work, as a more promising alternative to the naive diagonal approximation.
Comparing the submodular function method SubLoRA-G and the exact Hessian method HessLoRA-G, we find that their performances are comparable, with the exact Hessian method occasionally exhibiting slightly better results. 
This can be attributed to the fact that the exact Hessian often approximately satisfies the projection condition in \Cref{eq_proj_hessian}, making the resulting objective function nearly submodular, and the greedy algorithm performs well without the Hessian projection.
For example, in \Cref{fig_static_elloptic_5_loss} and \Cref{fig_static_elloptic_5_error}, the learning curves of both methods perform closely, indicating that the projected Hessian is effectively identical to the original Hessian and that submodularity is already present.

We also analyze and compare the computational costs of different rank determination methods. The CPU/GPU runtimes are summarized in \Cref{table_runtime}. Notably, SubLoRA-G does not introduce significant overhead compared to LinearLoRA and DiagLoRA, demonstrating its efficiency. 
We also observe that the runtime of SubLoRA-G remains relatively stable once the rank budget exceeds a certain threshold. This indicates that our method is robust with respect to the budget size, as the computational cost does not grow significantly with larger budgets.
SubLoRA-R requires slightly higher runtime due to repeated random number generation in each iteration of the randomized greedy algorithm (\Cref{alg_randomized_greedy}). Furthermore, since rank determination is performed only once, the additional computational cost introduced by SubLoRA is negligible compared to the overall LoRA fine-tuning runtime, which is approximately 5 minutes. Specifically, Stage 1 of \Cref{alg_rank_determination}, which involves standard LoRA fine-tuning, takes around 5 minutes in total, while the time spent in Stages 2 and 3 for rank determination, as reported in \Cref{table_runtime}, is negligible (in the order of a few seconds) and does not significantly affect the overall computation time.

\begin{figure}[tb!]
    \centering
    \subfigure[$\blam=(1,1)$, training loss]{%
        \includegraphics[width=0.30\textwidth]{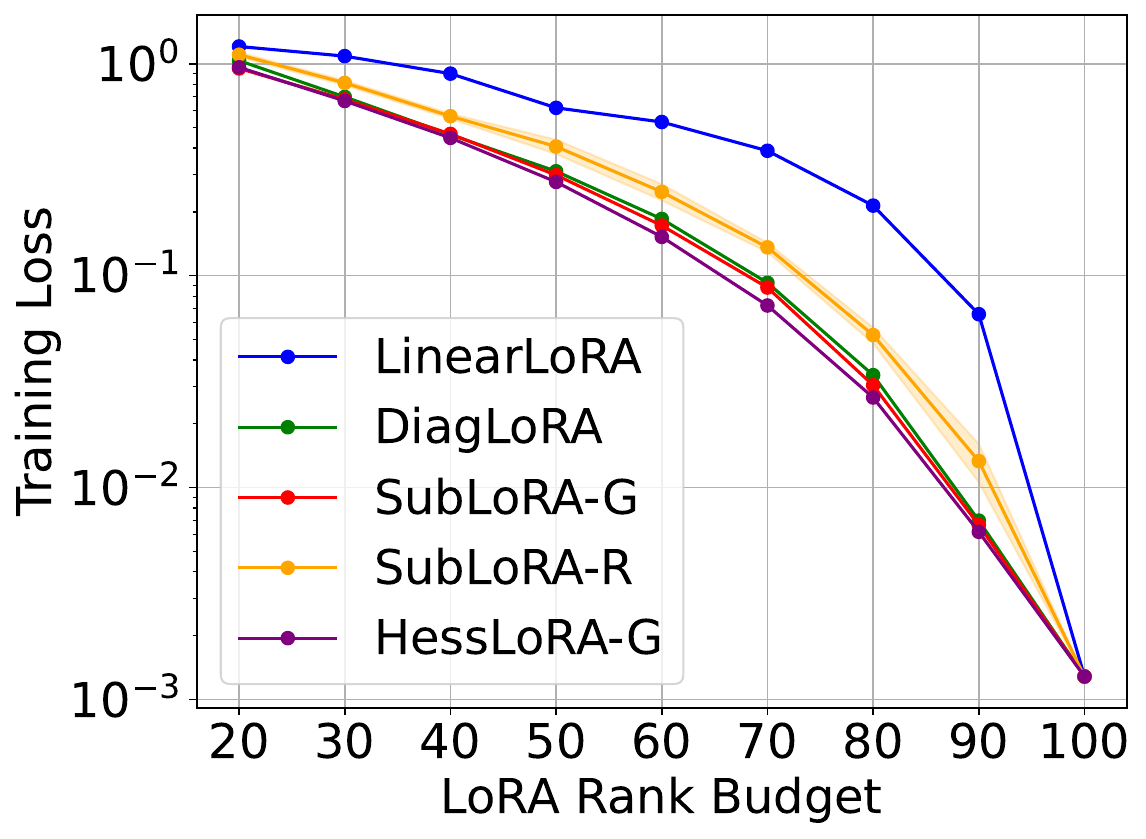} \label{fig_static_elloptic_1_loss}}
    \subfigure[$\blam=(1,5)$, training loss]{%
        \includegraphics[width=0.30\textwidth]{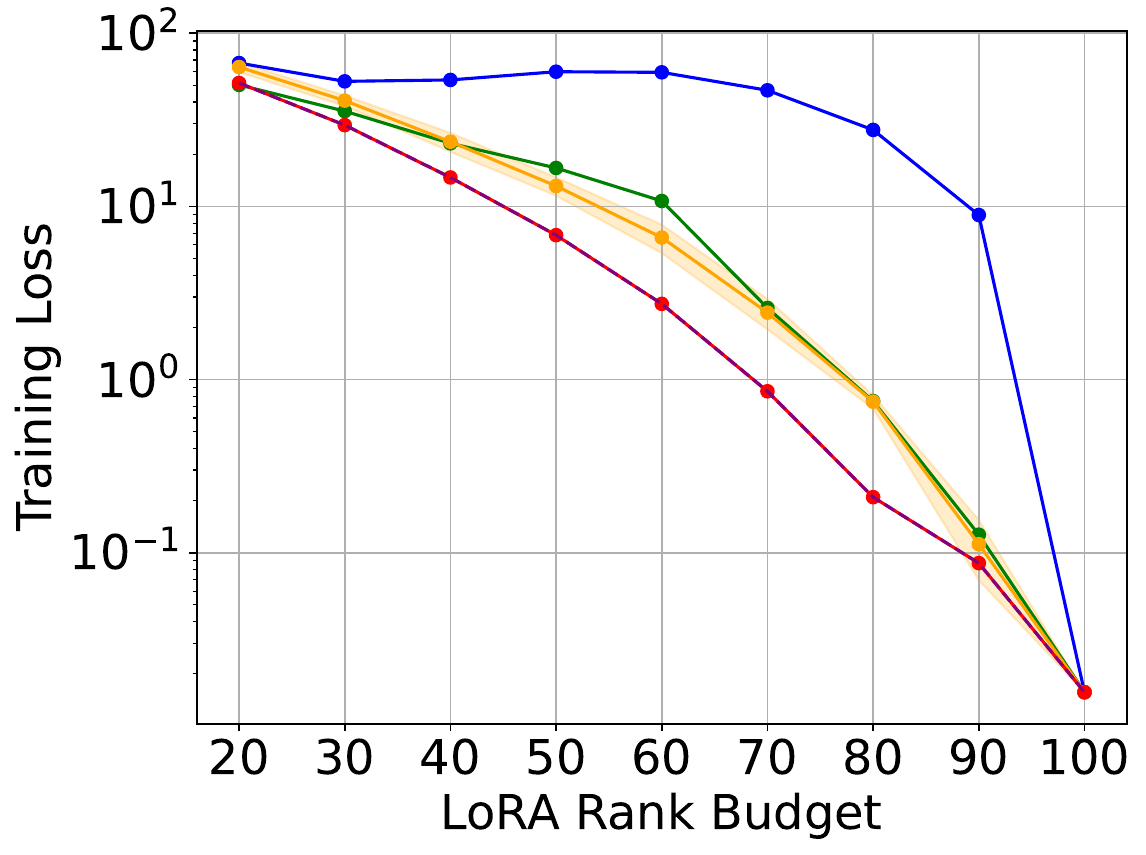} \label{fig_static_elloptic_5_loss}}
    \subfigure[$\blam=(2,1)$, training loss]{%
        \includegraphics[width=0.30\textwidth]{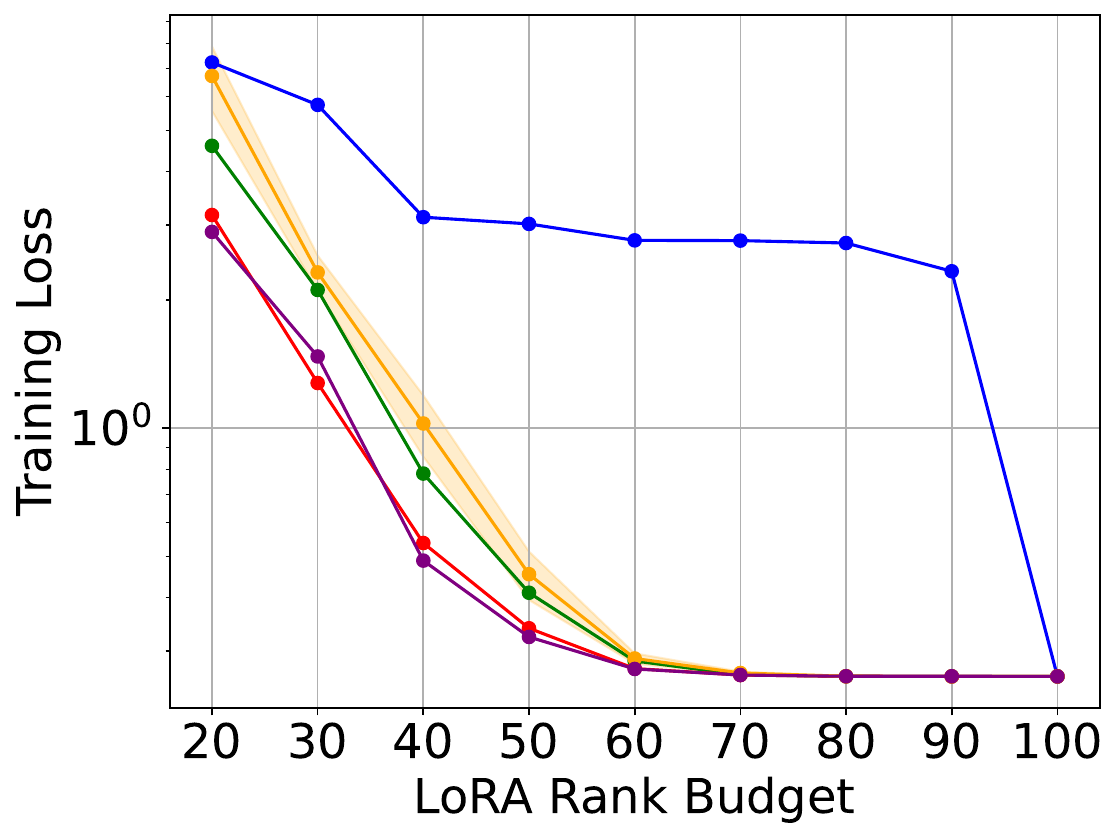} \label{fig_static_elloptic_1_2_loss}}\\
    \subfigure[$\blam=(1,1)$, relative error]{%
        \includegraphics[width=0.30\textwidth]{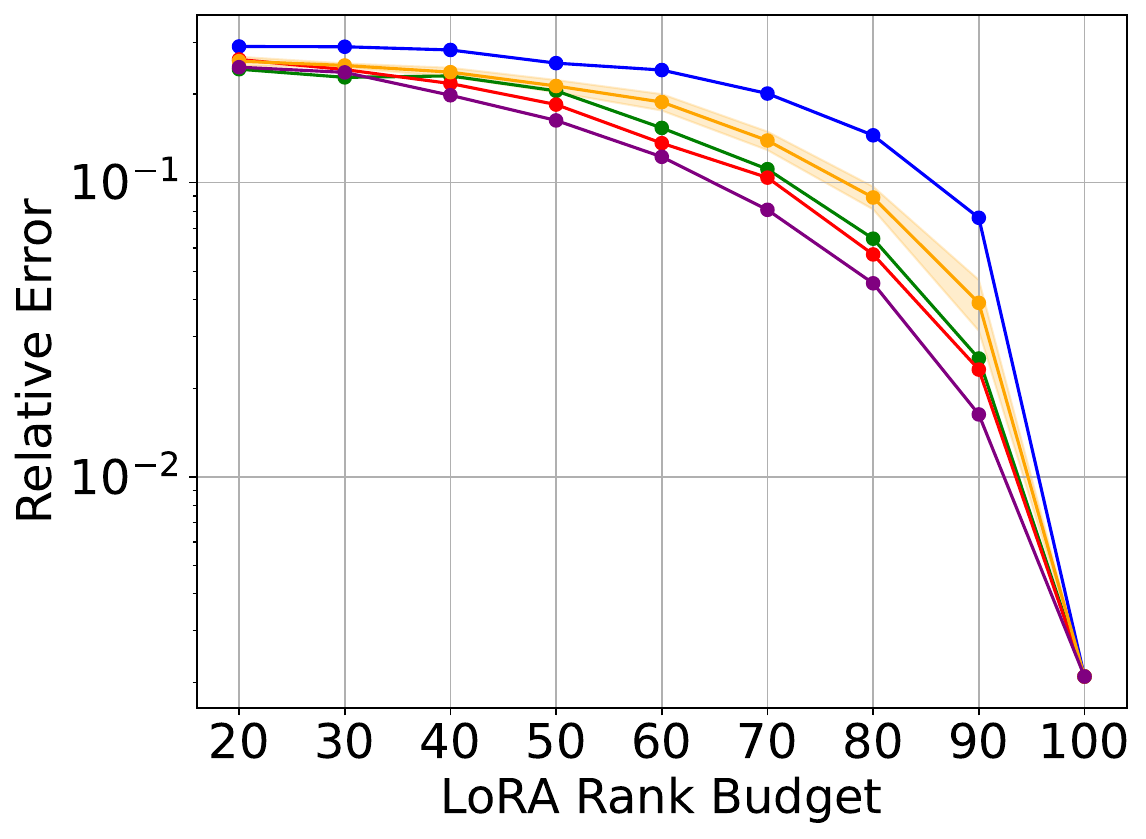} \label{fig_static_elloptic_1_error}}
    \subfigure[$\blam=(1,5)$, relative error]{%
        \includegraphics[width=0.30\textwidth]{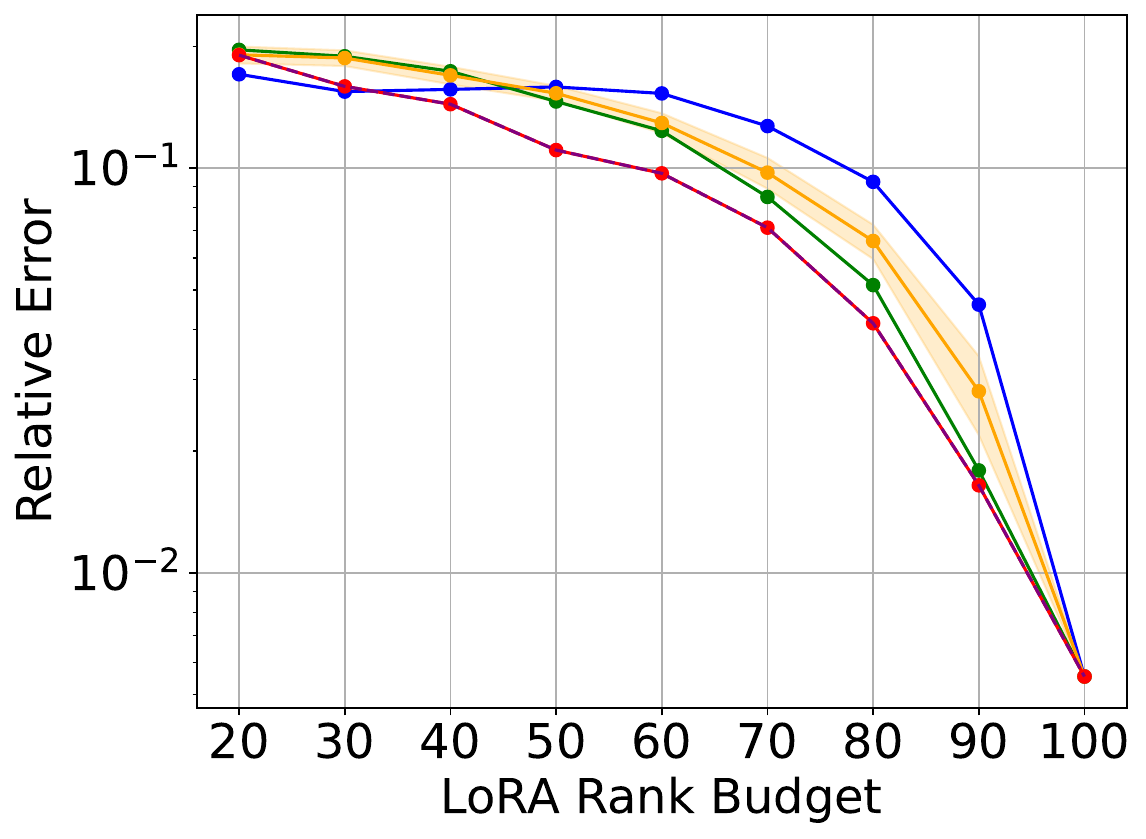} \label{fig_static_elloptic_5_error}}
    \subfigure[$\blam=(2,1)$, relative error]{%
        \includegraphics[width=0.30\textwidth]{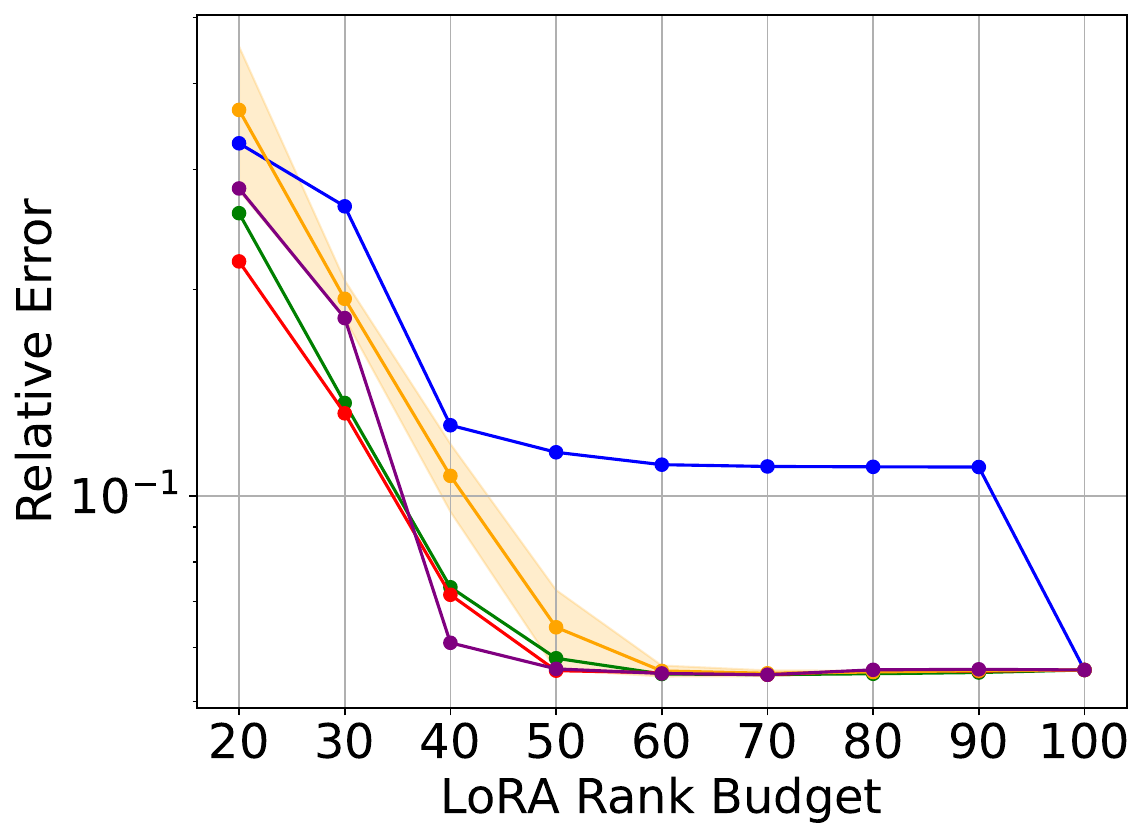} \label{fig_static_elliptic_1_2_error}}
    \caption{Performance comparison of different rank determination methods on elliptic equations under varying LoRA rank budgets and physical parameters. 
    }
    \label{fig_static_elliptic}
\end{figure}

\begin{table}[tp!]
\centering
\begin{tabular}{|l|l|l|l|l|l|l|l|l|l|}
\hline
\makecell{Methods \textbackslash Budgets $b$} & 20   & 30   & 40   & 50   & 60   & 70   & 80   & 90   & 100  \\ \hline
LinearLoRA                   & 0.53 & 0.54 & 0.53 & 0.53 & 0.54 & 0.53 & 0.53 & 0.53 & 0.52 \\ \hline
DiagLoRA                     & 1.39 & 1.42 & 1.44 & 1.42 & 1.40  & 1.40  & 1.42 & 1.39 & 1.40  \\ \hline
SubLoRA-G                    & 2.52 & 2.99 & 3.40  & 3.70  & 4.06 & 4.21 & 4.37 & 4.47 & 4.50  \\ \hline
SubLoRA-R                    & 4.88 & 5.15 & 6.04 & 6.64 & 7.11 & 7.58 & 7.72 & 7.91 & 8.02 \\ \hline
HessLoRA                     & 2.49 & 2.96 & 3.27 & 3.60  & 3.88 & 4.17 & 4.25 & 4.36 & 4.44 \\ \hline
\end{tabular}
\caption{Runtime comparison (seconds) across different rank determination methods on elliptic equations under varying LoRA rank budgets and $\blam=(1,1)$. We observe that any increase in runtime (if any) is minimal, even if the budget is increased. This highlights the efficiency of the greedy and the randomized greedy procedures.}
\label{table_runtime}
\end{table}

\subsubsection{Allen--Cahn Equations}
\label{sec_allen_cahn}
We consider a class of Allen--Cahn equations with varying physical parameters $\blam \in \mathbb{R}^2$:
\begin{equation}
    \begin{split}
        \frac{\partial u(t,\bm{x};\blam)}{\partial t} & - \Delta u(t,\bm{x};\blam) -u(t,\bm{x};\blam)^3 + u(t,\bm{x};\blam)\\
        & = g(t,\bm{x};\blam), \quad (t,\bm{x}) \in [0,1] \times \Omega,\\
        u(t,\bm{x};\blam) & = h_{1}(t,\bm{x};\blam), \quad (t,\bm{x}) \in [0,1] \times \partial \Omega,\\
        u(0,\bm{x};\blam) & = h_{2}(\bm{x};\blam), \quad \bm{x} \in \Omega,
    \end{split}
\end{equation}
where the temporal and  spatial domains are defined as $[0,1]$ and $\Omega = \{\bx \in \mathbb{R}^{2}: \left\|\bx \right\|_2 \leq 1\}$, respectively. Similarly, the exact solution $u(\bx;\blam)$ with the parameter $\blam$ is defined as $u(\bx;\blam) = e^{-t} \sin \left(\frac{\pi \lambda_1}{2}\left( 1 - \left\|\bm{x} \right\|_2\right)^{2.5} \right) + \lambda_2 \cdot e^{-t} \sin \left(\frac{\pi}{2}\left( 1 - \left\|\bm{x} \right\|_2\right) \right)$. 
We first pre-train an MLP model on the PDE with $\blam =(1,0)$, and then fine-tune it by standard LoRA on PDEs with $\blam=(1,1)$, $(1,5)$, and $(2,1)$. 
We adopt the same experimental setup and rank determination methods for comparison as described in \Cref{sec_elliptic}.

We visualize the experimental results in \Cref{fig_static_allen-cahn}. Similar trends to those observed in the elliptic equation experiments are evident here, further validating the effectiveness of incorporating second-order information compared to using only first-order approximations. 
Notably, in these examples, SubLoRA-G slightly outperforms HessLoRA-G. This highlights the advantage of transforming the objective into a submodular function by the Hessian projection defined in \Cref{eq_projection_submodular}.
It is important to note that a general (non-submodular) quadratic objective is not guaranteed to perform well under the greedy algorithm, where the performance can degrade significantly in certain cases. However, by projecting the Hessian to enforce the submodularity, as done in \Cref{eq_projection_submodular}, we ensure that applying the greedy algorithm becomes theoretically sound and practically robust. Therefore, the use of Hessian projection followed by greedy algorithm, as implemented in \Cref{alg_rank_determination}, is a reasonable and effective strategy for rank determination.

\begin{figure}[tb!]
    \centering
    \subfigure[$\blam=(1,1)$, training loss]{%
        \includegraphics[width=0.30\textwidth]{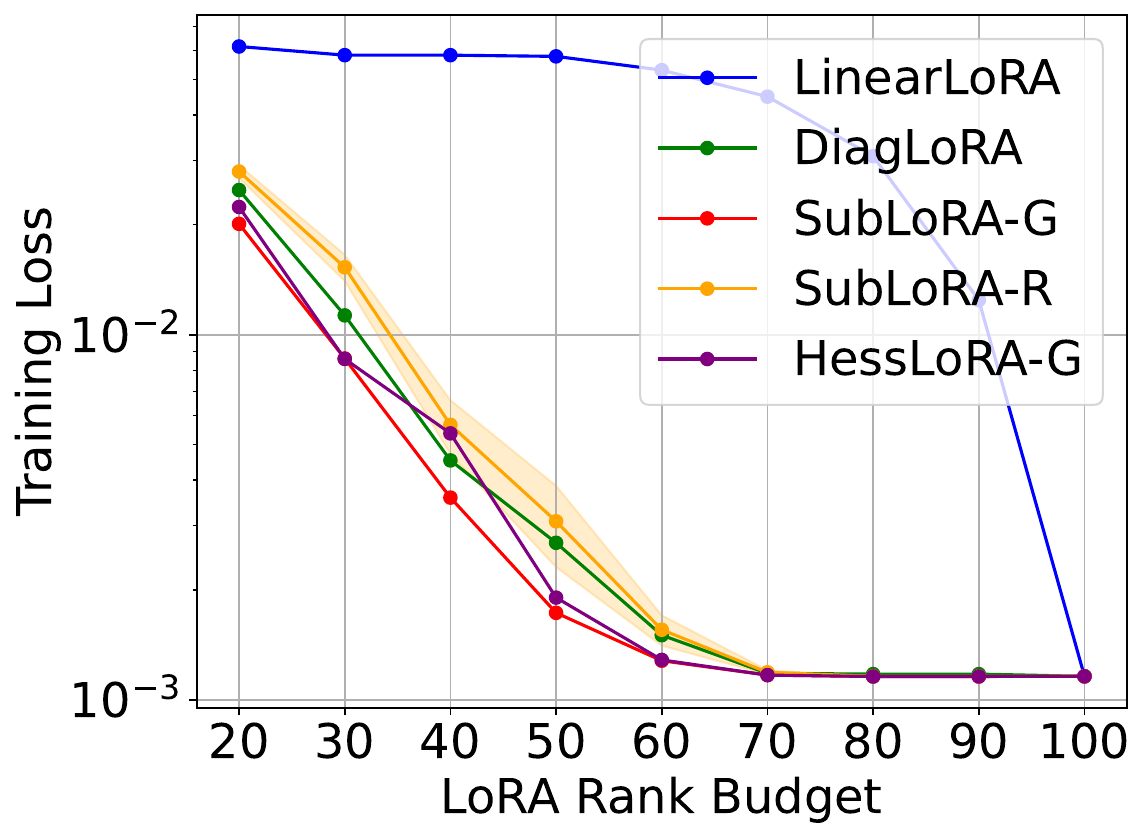} \label{fig_static_allen-cahn_1_loss}}
    \subfigure[$\blam=(1,5)$, training loss]{%
        \includegraphics[width=0.30\textwidth]{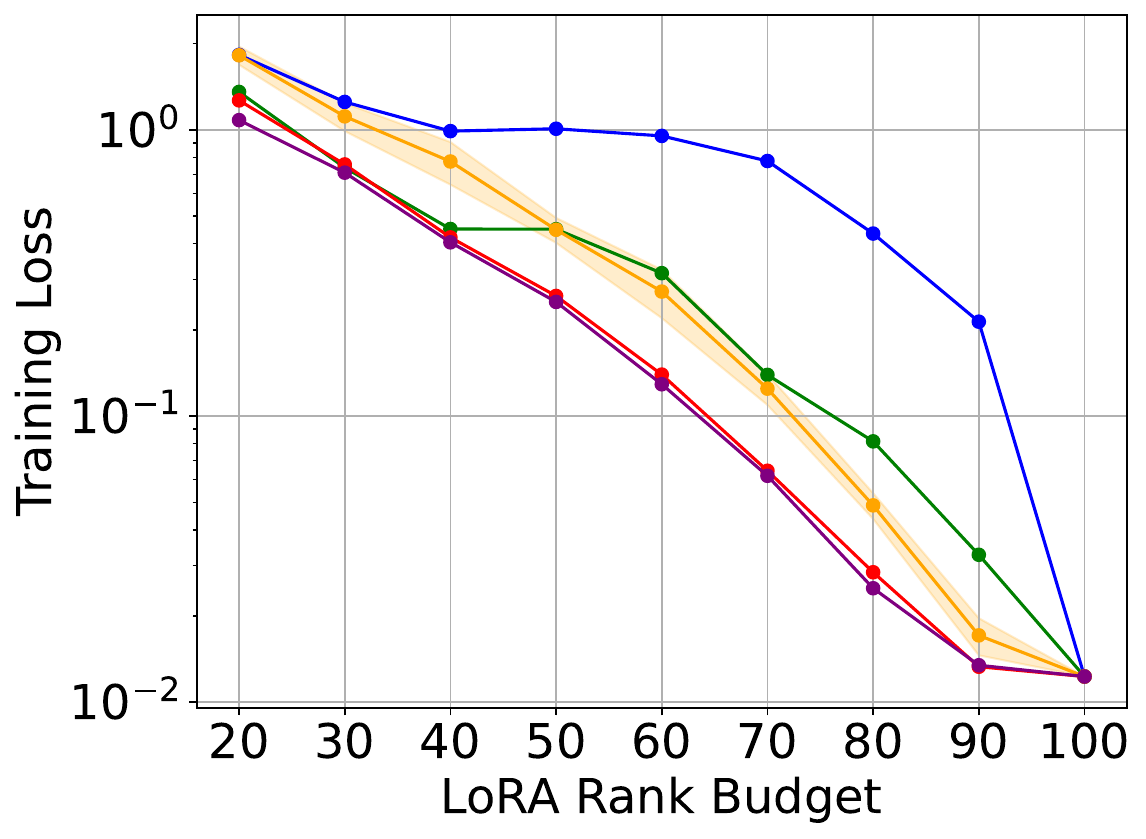} \label{fig_static_allen-cahn_5_loss}}
    \subfigure[$\blam=(2,1)$, training loss]{%
        \includegraphics[width=0.30\textwidth]{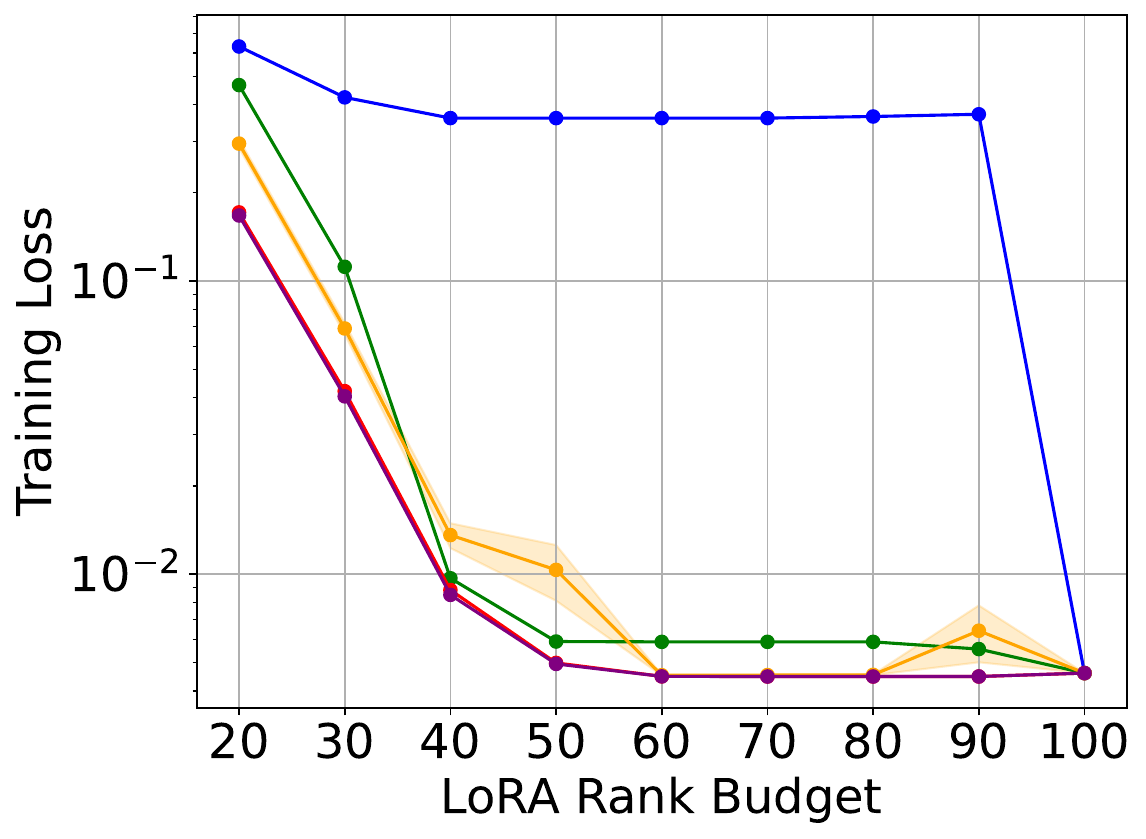} \label{fig_static_allen-cahn_1_2_loss}}\\
    \subfigure[$\blam=(1,1)$, relative error]{%
        \includegraphics[width=0.30\textwidth]{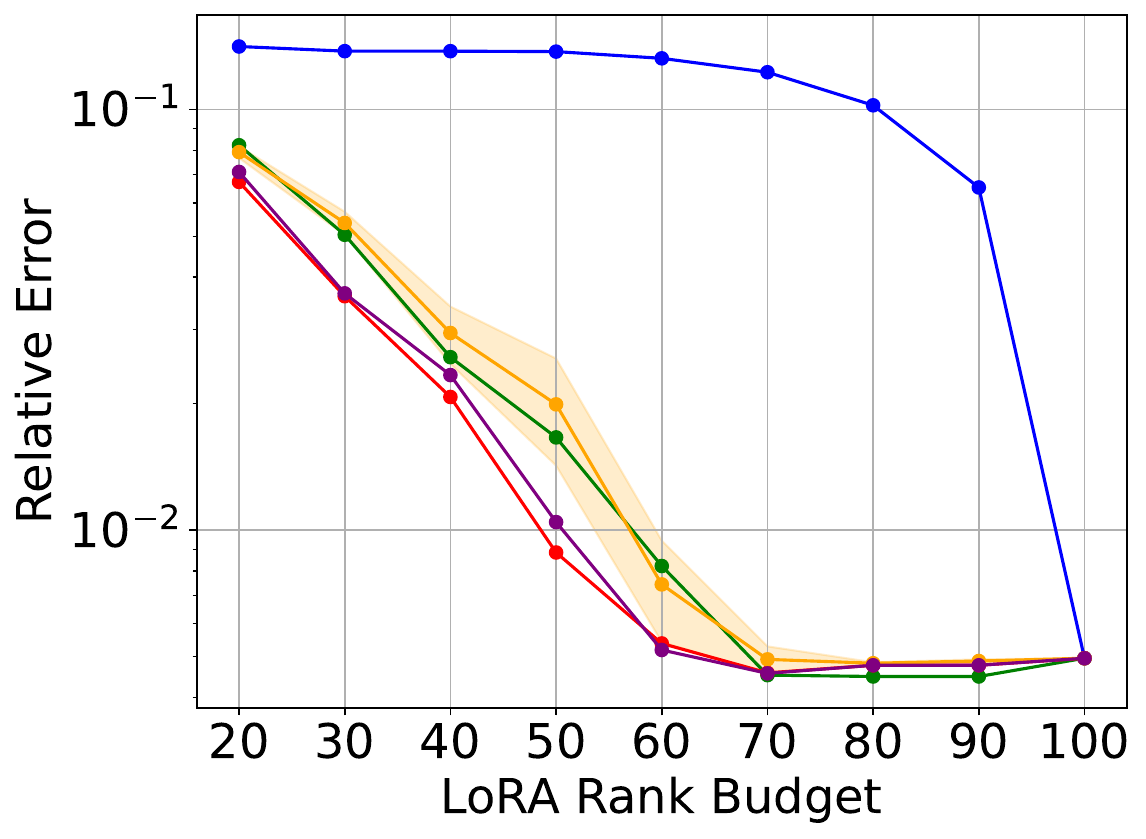} \label{fig_static_allen-cahn_1_error}}
    \subfigure[$\blam=(1,5)$, relative error]{%
        \includegraphics[width=0.30\textwidth]{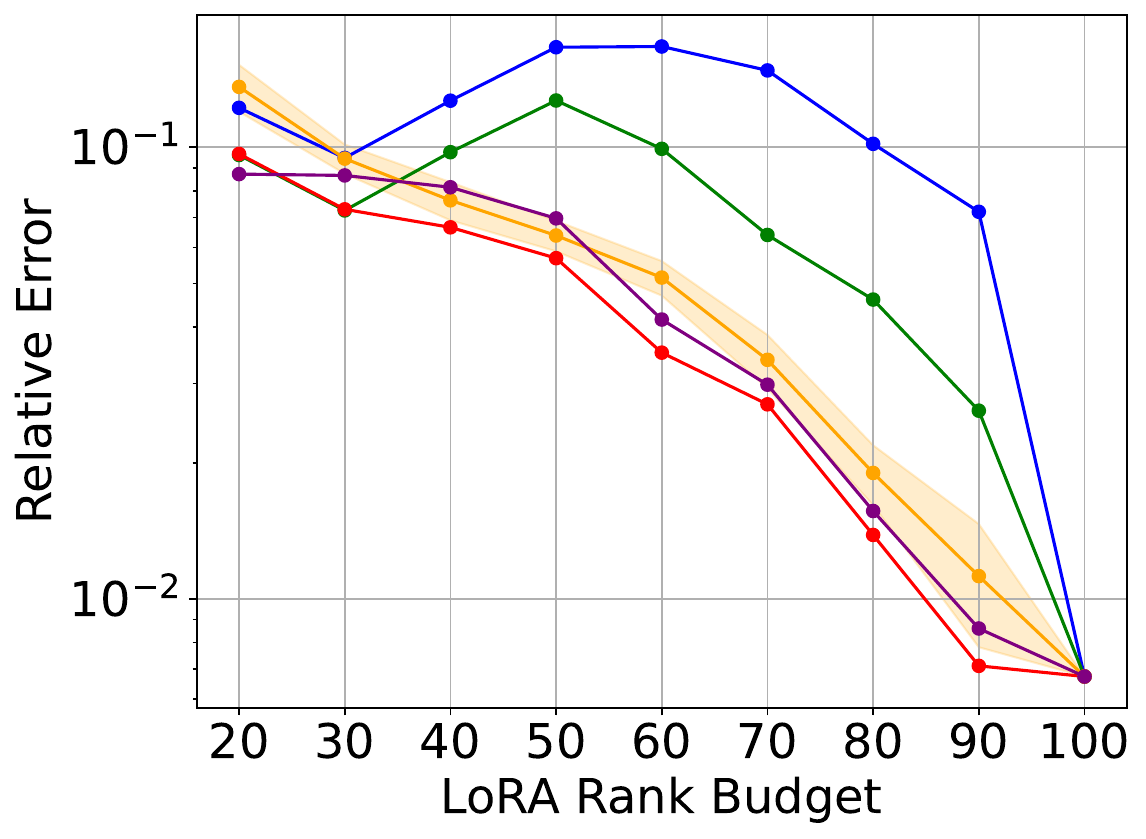} \label{fig_static_allen-cahn_5_error}}
    \subfigure[$\blam=(2,1)$, relative error]{%
        \includegraphics[width=0.30\textwidth]{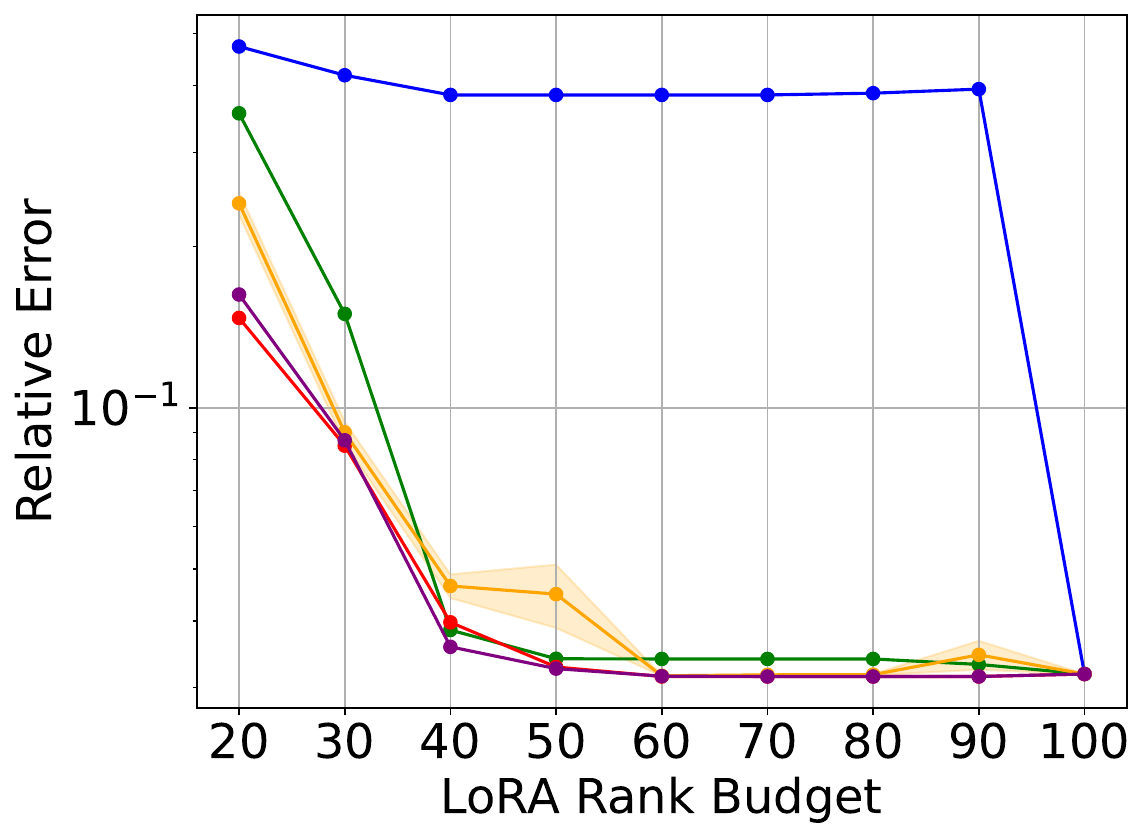} \label{fig_static_allen-cahn_1_2_error}}
    \caption{
    Performance comparison of different rank determination methods on Allen--Cahn equations under varying LoRA rank budgets and physical parameters. 
    }
    \label{fig_static_allen-cahn}
\end{figure}

\subsubsection{Hyperbolic Equations}
We consider a class of hyperbolic equations with varying physical parameters $\blam \in \mathbb{R}^{2}$:
\begin{align}
        \frac{\partial^2 u(t,\bm{x};\blam)}{\partial t^2} - \Delta u(t,\bm{x};\blam) & = g(t,\bm{x};\blam), \quad (t,\bm{x}) \in [0,1] \times \Omega, \nonumber\\
        u(t,\bm{x};\blam) & = h_1(t,\bm{x};\blam), \quad (t,\bm{x}) \in [0,1] \times \partial \Omega,\nonumber\\
        u(0,\bm{x};\blam) & = h_2(\bm{x};\blam), \quad \bm{x} \in \Omega,\nonumber\\
        \frac{u(0,\bm{x};\blam)}{\partial t} & = h_3(\bm{x};\blam), \quad \bm{x} \in \Omega,
\end{align}
where the temporal and the spatial domains are defined as $[0,1]$ and $\Omega = \{\bx \in \mathbb{R}^{2}: \left\|\bx \right\|_2 \leq 1\}$, respectively. The exact solution $\bu(\bx;\blam)$ with the parameter $\blam$ is defined as $u(\bx;\blam) = \big(e^{t^2}-1\big) \sin \big(\frac{\pi \lambda_1}{2}\big( 1 - \left\|\bm{x} \right\|_2 \big)^{2.5} \big) + \lambda_2 \cdot \big(e^{t^2}-1\big) \sin \big(\frac{\pi}{2}\big( 1 - \left\|\bm{x} \right\|_2\big ) \big)$. 
We adopt experimental settings similar to those in \Cref{sec_elliptic} and \Cref{sec_allen_cahn}, including model architecture, pre-training setup, LoRA fine-tuning procedures, and rank determination baselines for comparison.

We observe similar performance trends for hyperbolic equations as those seen in the elliptic and Allen--Cahn cases, and thus omit redundant details here. 
In \Cref{fig_static_hyperbolic_1_2_loss} and \Cref{fig_static_hyperbolic_1_2_error}, DiagLoRA performs comparably to SubLoRA-G. We verify that this is primarily because the Hessian matrix in this example is diagonally dominant, where its diagonal entries capture most of the curvature information, while the off-diagonal terms contribute marginally. 
It also explains the close alignment between the performance of SubLoRA-G and HessLoRA-G, where the projection in \Cref{eq_proj_hessian} for the off-diagonal elements of Hessian makes a less significant effect.

\begin{figure}[tb!]
    \centering
    \subfigure[$\blam=(1,1)$, training loss]{%
        \includegraphics[width=0.30\textwidth]{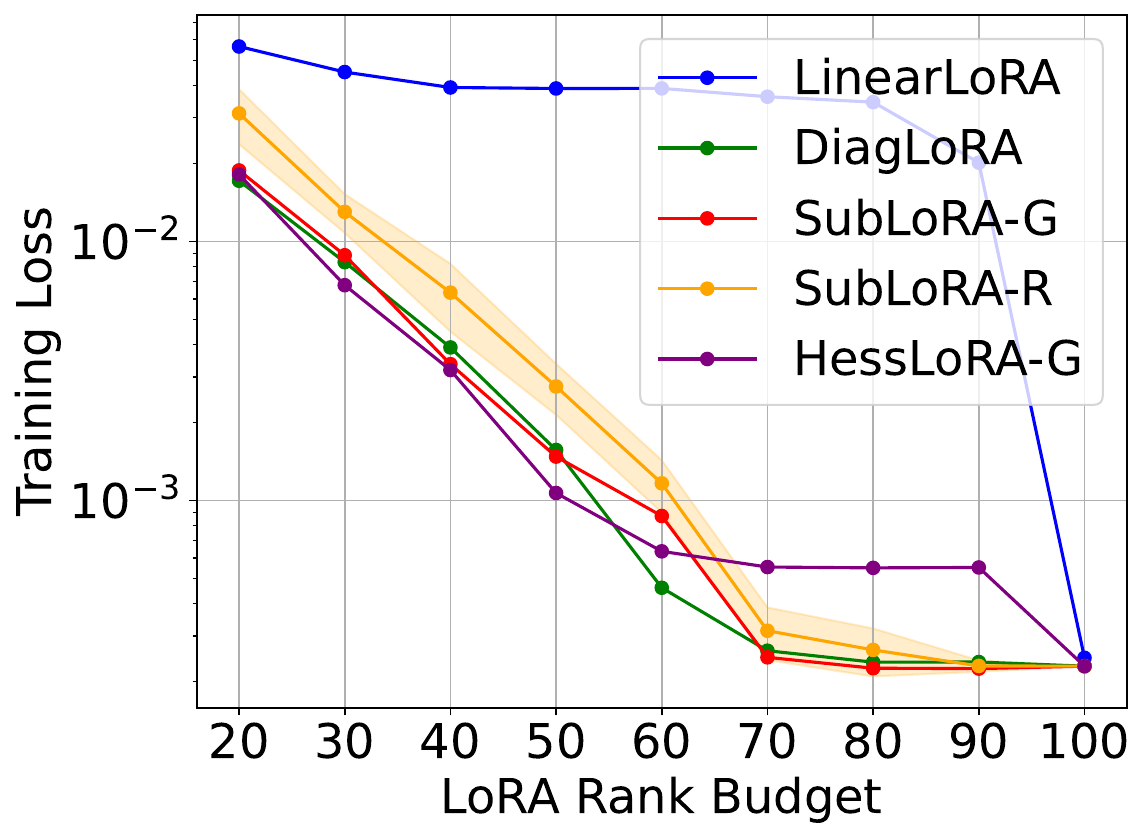} \label{fig_static_hyperbolic_1_loss}}
    \subfigure[$\blam=(1,5)$, training loss]{%
        \includegraphics[width=0.30\textwidth]{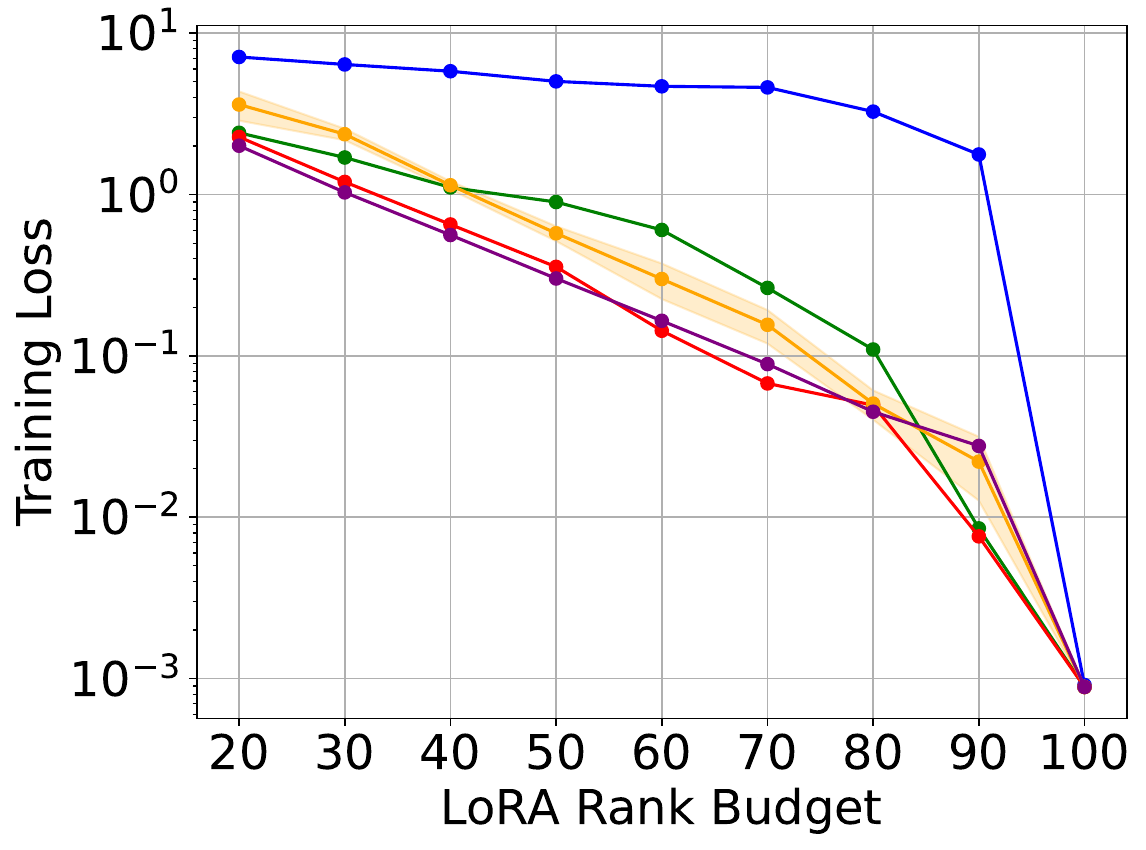} \label{fig_static_hyperbolic_5_loss}}
    \subfigure[$\blam=(2,1)$, training loss]{%
        \includegraphics[width=0.30\textwidth]{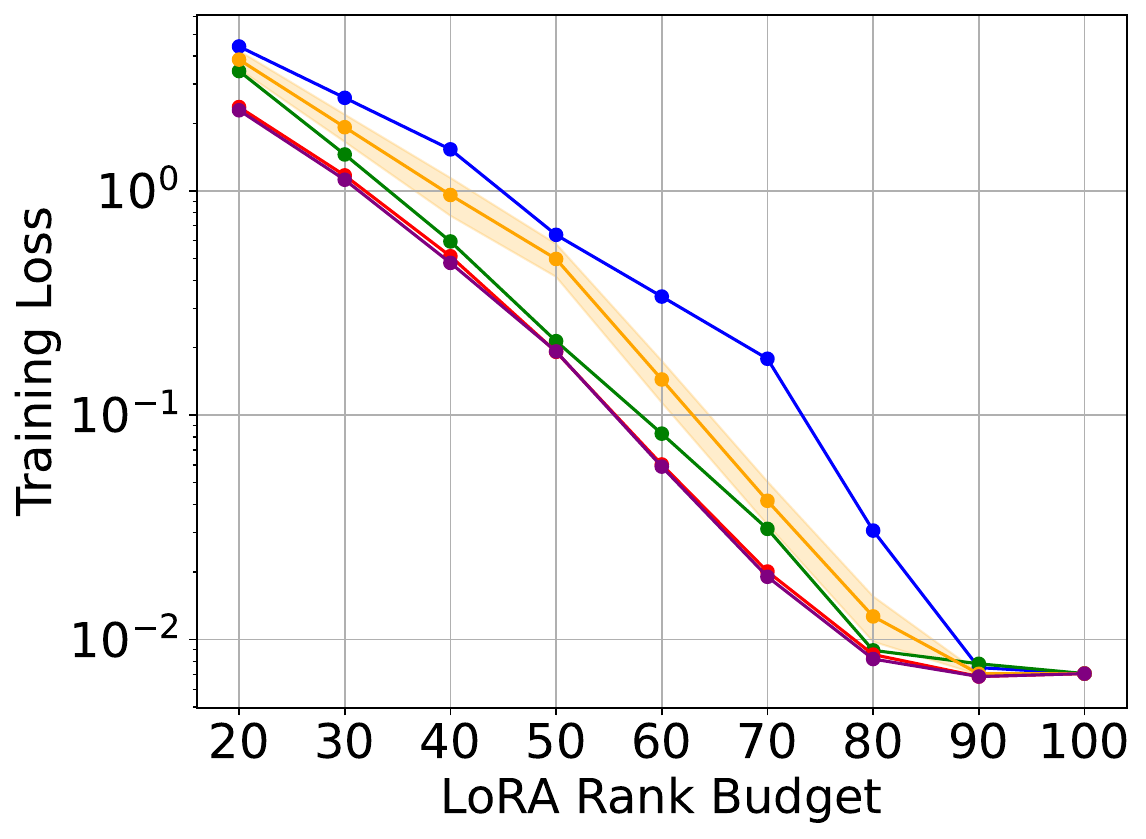} \label{fig_static_hyperbolic_1_2_loss}}\\
    \subfigure[$\blam=(1,1)$, relative error]{%
        \includegraphics[width=0.30\textwidth]{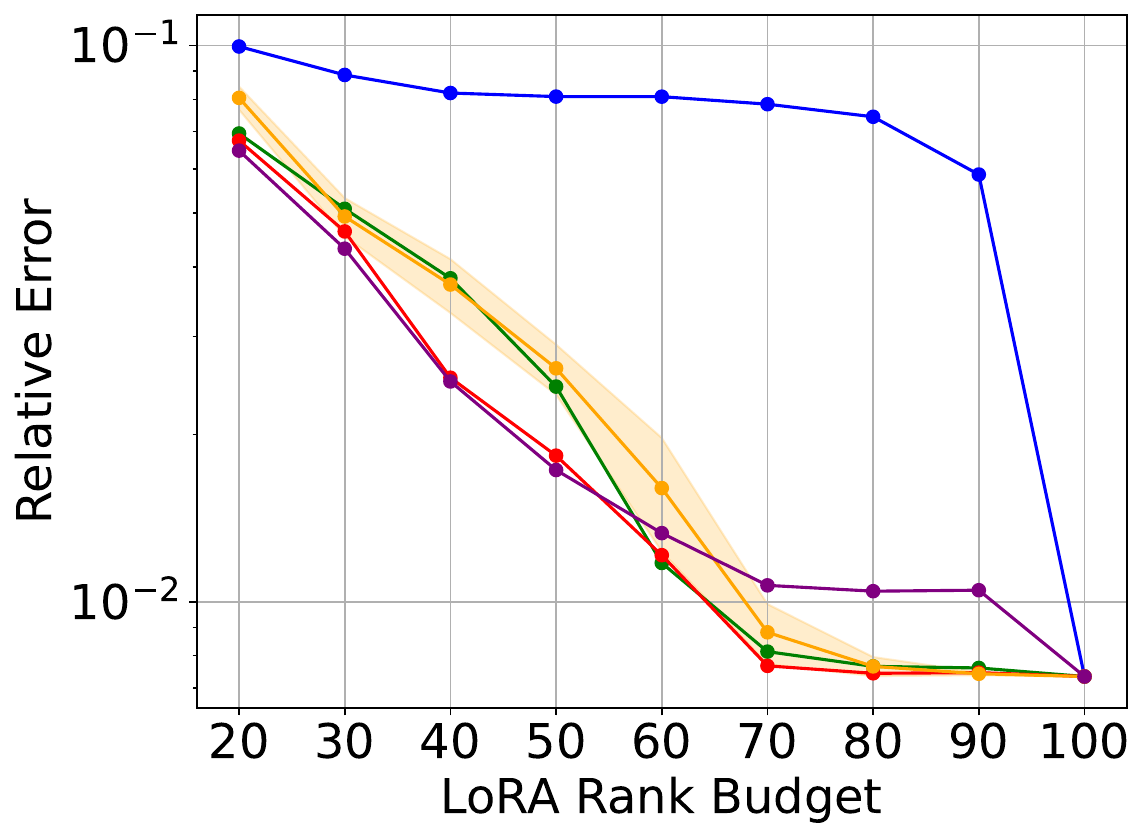} \label{fig_static_hyperbolic_1_error}}
    \subfigure[$\blam=(1,5)$, relative error]{%
        \includegraphics[width=0.30\textwidth]{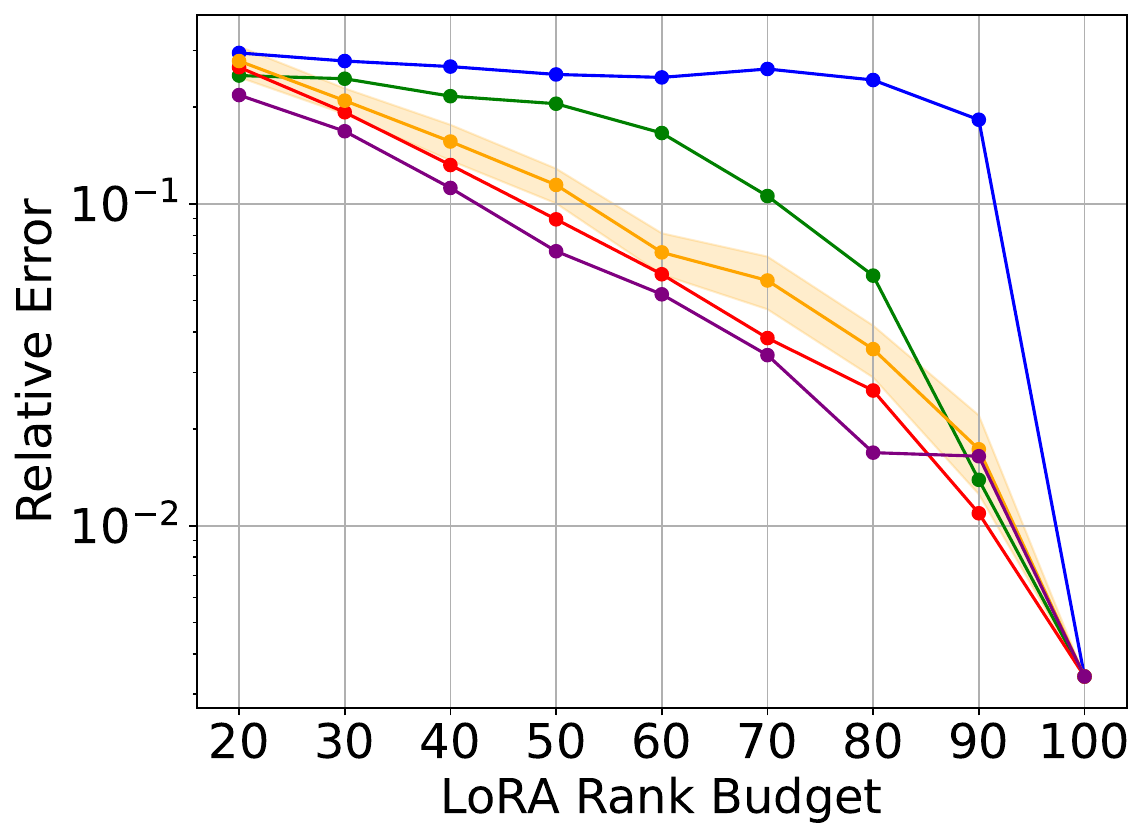} \label{fig_static_hyperbolic_5_error}}
    \subfigure[$\blam=(2,1)$, relative error]{%
        \includegraphics[width=0.30\textwidth]{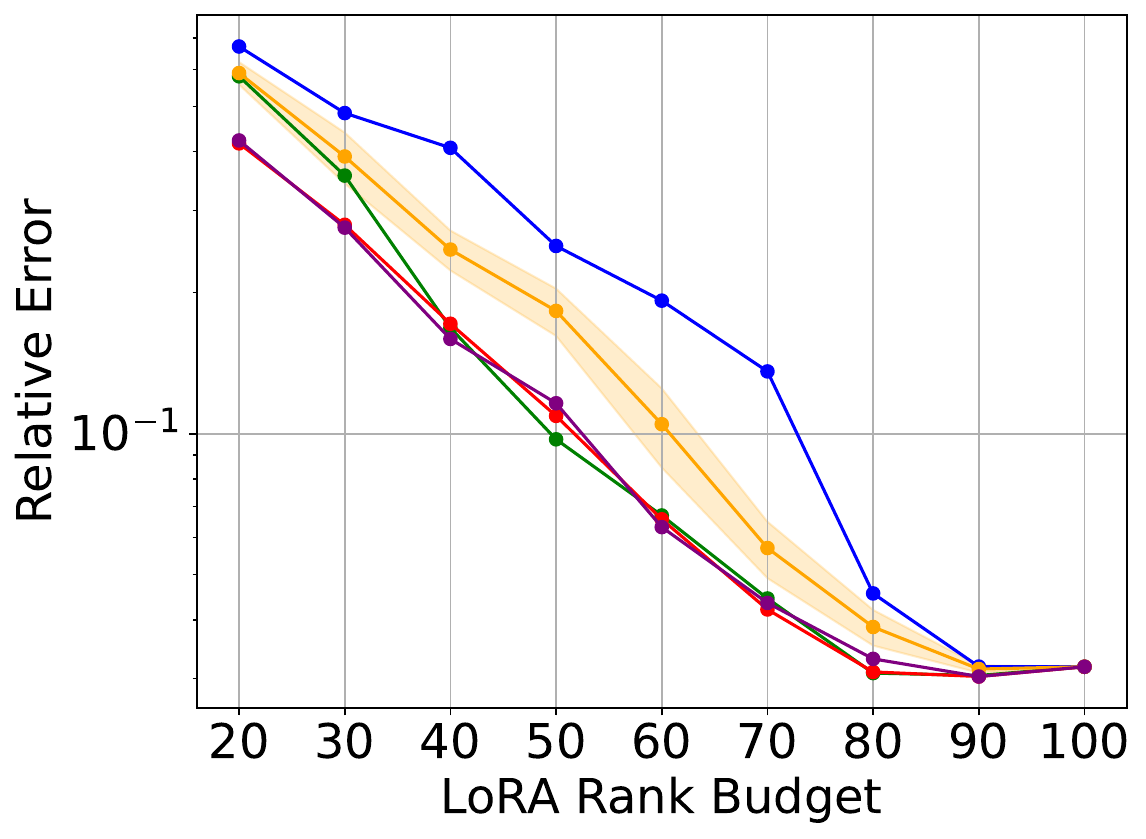} \label{fig_static_hyperbolic_1_2_error}}
    \caption{
    Performance comparison of different rank determination methods on hyperbolic equations under varying LoRA rank budgets and physical parameters. 
    }
    \label{fig_static_hyperbolic}
\end{figure}

\subsection{Alternating Training with Rank Determination}
We also evaluate the alternating algorithm for LoRA fine-tuning and rank determination, as described in \Cref{alg_lora_update_rank_determination}, on a class of PDEs with varying physical parameters (detailed in \Cref{sec_applications}). The same PDE types and tasks used in \Cref{sec_experiments_training_free} are considered.
We begin by pre-training neural networks (MLPs) to obtain the pre-trained parameters $\Theta_{\text{pt}}$. Then, we apply \Cref{alg_lora_update_rank_determination} to alternate between LoRA parameter updates and rank determination.
In all experiments, we use a four-layer MLP (three hidden layers) as the base architecture for PINN training. Due to the low-dimensional input and output, LoRA fine-tuning is applied only to the hidden layers. Each hidden layer has a width of $1000$, and LoRA is initially applied with a rank of $50$ per layer, giving a total initial LoRA rank of $100$. However, the total rank budget is set to $40$, requiring allocation across the model’s layers.
During each outer-loop iteration, Stage 1 of \Cref{alg_lora_update_rank_determination} performs LoRA parameter updates using the Adam optimizer for 100 epochs. In Stages 2 and 3, rank determination methods are applied to prune less informative singular values, ensuring that the number of non-zero singular values remains within the rank budget. This alternating procedure is repeated for $T = 5$ iterations.

The performance of the alternating algorithms using different rank determination methods is summarized in \Cref{tab_compare_alternating}.
We observe that the proposed alternating SubLoRA method based on submodular function maximization consistently achieves the lowest training loss and validation error. 
These results highlight the effectiveness of incorporating Hessian information, which enables a more accurate characterization of the complex loss landscape. Furthermore, the Hessian projection used to construct a submodular objective allows the greedy algorithm to serve as a reliable solver for the formulated combinatorial optimization problem.

To further illustrate the training dynamics, we visualize the optimization trajectories of \Cref{alg_lora_update_rank_determination} in \Cref{fig_training_allen-cahn_1.0_loss_alternating}, comparing the alternating LinearLoRA (first-order) and the alternating SubLoRA (second-order) methods. In the figure, circular markers represent the outputs of Stage 1 (parameter updates), while cross markers indicate the results after rank determination in Stages 2 and 3.
The visualization shows that the proposed SubLoRA method facilitates more accurate rank determination, which in turn guides Stage 1 to progressively converge to parameter configurations that better align with the rank budget. This interaction leads to stable convergence of the alternating process. In contrast, the first-order method of LinearLoRA fails to decide LoRA ranks effectively, often discarding important components, thereby disrupting learning and preventing convergence to a meaningful low-rank structure.

We also conduct ablation studies on the LoRA rank budget $b$ within the alternating SubLoRA algorithm. The training dynamics for solving the Allen--Cahn equations with $\blam=(1,1)$ are shown in \Cref{fig_training_allen-cahn_1.0_loss_alternating_ablation_b}. When the budget is set to $b=20$, we observe a slow convergence or even divergence and a noticeably higher training loss. This is because the rank budget underestimates the capacity needed to represent the new PDE solution, leading to limited expressiveness and suboptimal fine-tuning. Increasing the budget to $b=40$ results in a significant improvement, where SubLoRA converges reliably and achieves a much lower training loss. Further increasing the budget to $b=80$ does not lead to additional gains, with the performance closely matching that of $b=40$. This suggests that $b=80$ is an overestimation and that a budget of $b=40$ is sufficient to capture the difference between the source and target PDE solutions. These observations justify our choice of $b=40$ in the main experiments. Moreover, the comparable performance between $b=40$ and $b=80$ highlights the robustness of SubLoRA to overestimated budgets, effectively identifying and retaining only the essential components while discarding redundant ones.

Note that in each outer iteration of the joint training procedure, alternating between LoRA parameter updates and rank determination, the rank determination step is performed only once. This indicates that the majority of the computational cost arises from the LoRA fine-tuning itself, rather than from rank determination. As shown in \Cref{table_runtime}, SubLoRA introduces only a marginal increase in runtime  (approximately four seconds) compared to LinearLoRA and DiagLoRA. This overhead is negligible in the context of the full alternating algorithm (\Cref{alg_lora_update_rank_determination}). Therefore, our method achieves comparable efficiency to LinearLoRA and DiagLoRA while providing significantly improved performance.

\begin{table*}[htb!]
\centering
\begin{tabular}{|l|l|ll|ll|ll|}
\hline
\multirow{2}{*}{PDE Types} &
\multirow{2}{*}{Methods} & \multicolumn{2}{l|}{$\blam=(1,1)$}                                      &  \multicolumn{2}{l|}{$\blam=(1,5)$}                                     & \multicolumn{2}{l|}{$\blam=(2,1)$}  \\
\cline{3-8}
& & loss & rel (\%) & loss & rel (\%)  & loss & rel (\%) \\
\hline
\multirow{4}{*}{elliptic} & LinearLoRA & 2.56E-3 & 0.77 & 2.11E-2	& 1.81 & 4.82E-1 & 	7.65 \\
& DiagLoRA & 2.63E-3 & 0.79 & 7.56E-3 &	0.18 & 1.29E-1 & 5.96\\
& SubLoRA-G & 1.39E-3 & 0.41 & 6.26E-3 & 0.17 & 1.11E-1 & 5.43\\
& SubLoRA-R & 1.62E-3 &	0.48 & 1.08E-2 & 0.56 & 1.10E-1 & 5.48\\
\hline
\multirow{4}{*}{Allen--Cahn} & LinearLoRA & 7.90E-3 & 3.85 & 3.32E-1 & 6.81 & 8.05E-3 & 5.04\\
& DiagLoRA & 1.40E-3 & 0.60 & 1.82E-2 & 0.93 & 2.89E-3 & 3.27\\
& SubLoRA-G & 3.10E-3 & 0.55 & 1.00E-2 & 0.69 & 2.21E-3 & 3.21\\
& SubLoRA-R & 1.83E-3 & 1.01 & 1.38E-2 & 0.79 & 2.38E-3 & 3.22\\
\hline
\multirow{4}{*}{hyperbolic} & LinearLoRA & 2.72E-2 & 6.61 & 8.47E-1 & 15.58 & 1.11E-1 & 6.55\\
& DiagLoRA & 9.47E-4 & 2.07 & 1.34E-1 & 6.77 & 7.69E-2 & 5.83\\
& SubLoRA-G & 4.76E-4 & 1.11 & 3.43E-2 & 2.24 & 2.97E-2 & 4.36\\
& SubLoRA-R & 1.11E-3 &	2.15 & 9.05E-2 & 4.78 & 8.55E-2 & 6.28\\
\hline
\end{tabular}
\caption{Loss and relative error (rel) of the alternating algorithms for solving elliptic, Allen–Cahn, and hyperbolic equations under varying physical parameters.}
\label{tab_compare_alternating}
\end{table*}

\begin{figure}
    \centering
    \includegraphics[width=0.5\linewidth]{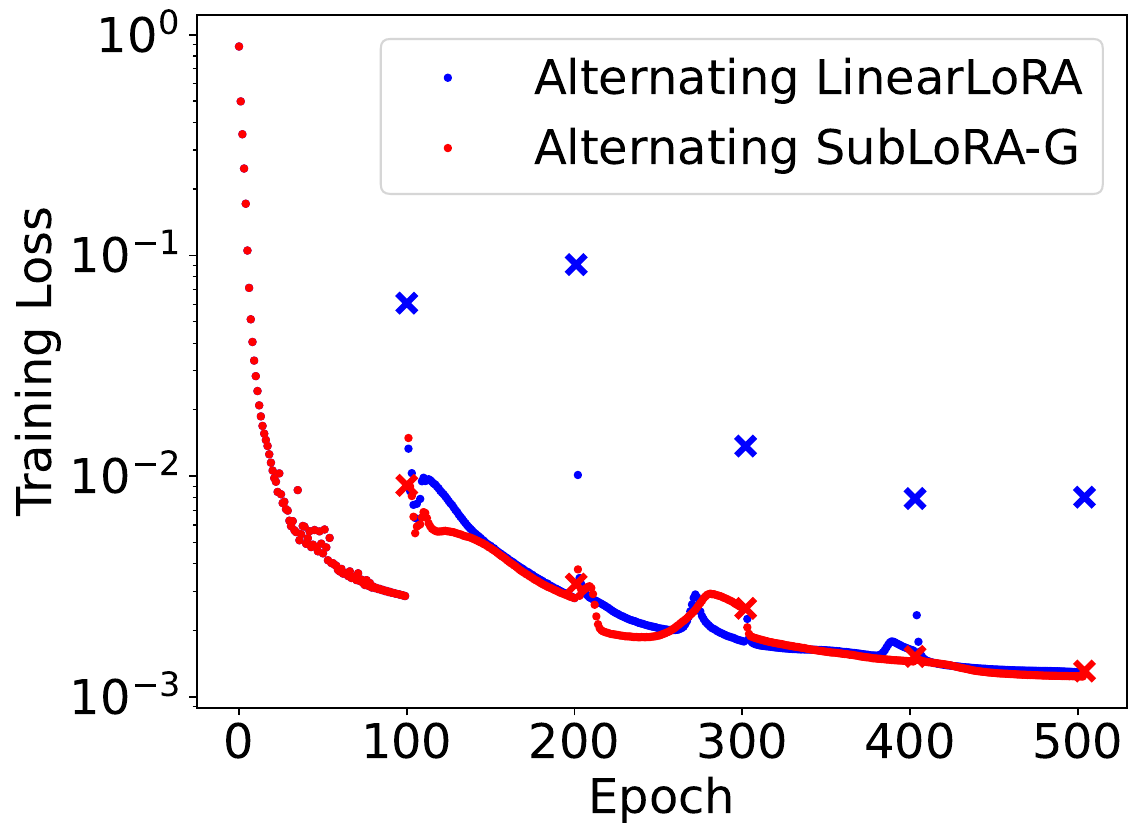}
    \caption{
    Training trajectories of \Cref{alg_lora_update_rank_determination} using (first-order) LinearLoRA and (second-order) SubLoRA-G as rank determination methods on Allen--Cahn equations with $\blam=(1,1)$. The SubLoRA-G method leads to more reliable rank determination, which effectively guides the alternating optimization to convergence. In contrast, the LinearLoRA method tends to discard critical components, resulting in divergence of the training process.
    }
    \label{fig_training_allen-cahn_1.0_loss_alternating}
\end{figure}

\begin{figure}
    \centering
    \includegraphics[width=0.5\linewidth]{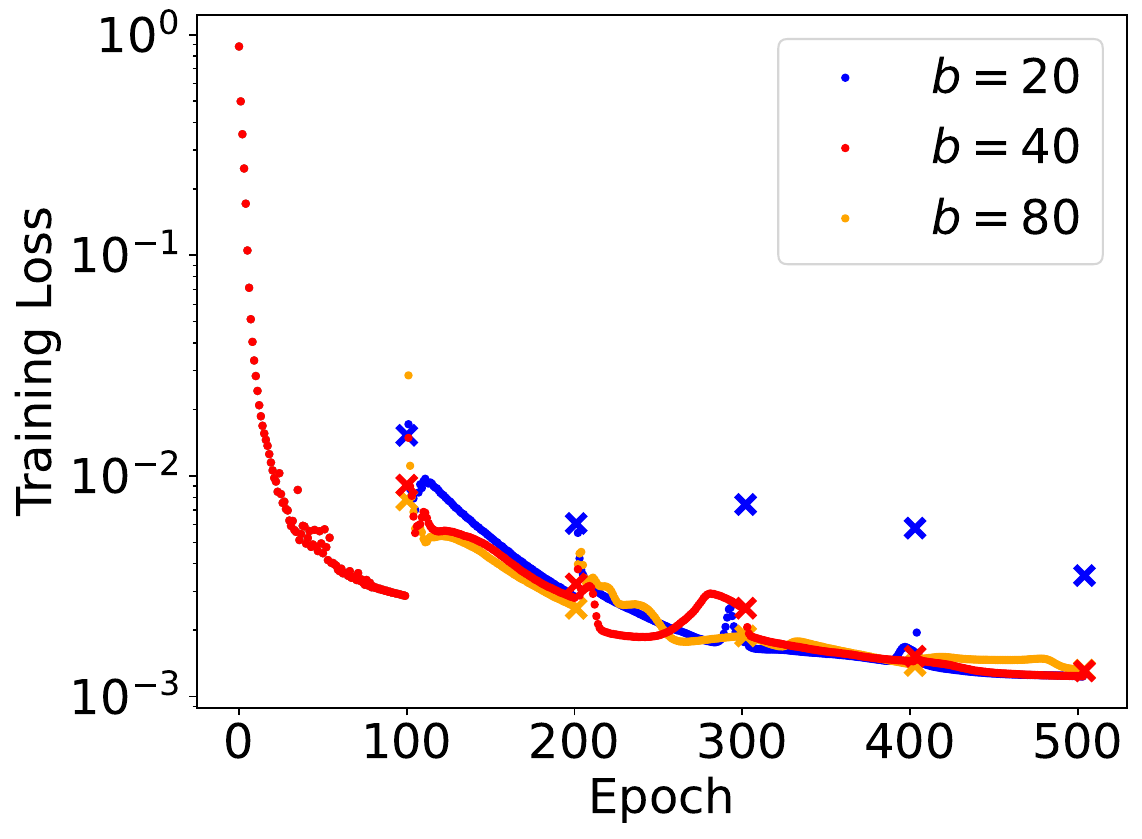}
    \caption{
    Training trajectories of \Cref{alg_lora_update_rank_determination} using SubLoRA-G as the rank determination method on Allen--Cahn equations with $\blam=(1,1)$ and varying LoRA rank budgets $b=20, 40, 80$. 
    A small budget ($b=20$) leads to underfitting and higher training loss due to limited expressiveness. Increasing the budget to $b=40$ leads to convergence and low training loss. Further increasing to $b=80$ results in similar performance, indicating that $b=40$ is sufficient. These results highlight SubLoRA’s robustness to overestimated budgets and its ability to allocate rank effectively.
    }
    
    \label{fig_training_allen-cahn_1.0_loss_alternating_ablation_b}
\end{figure}

\section{Conclusion}
\label{sec_conclusion}

In this paper, we introduce SubLoRA, a rank determination method for LoRA based on submodular function maximization. 
We formulate the rank determination problem as a combinatorial optimization problem with a set-valued quadratic objective derived from the second-order Taylor expansion of the fine-tuning loss. 
To address the computational challenges of optimizing this objective, we design a Hessian projection that transforms the objective into a submodular function. The transformation enables the use of efficient greedy algorithms with theoretical approximation guarantees, making the method both computationally tractable and effective.
To further enhance LoRA fine-tuning with rank determination, we propose an alternating algorithm that iteratively updates LoRA parameters and performs rank determination.
We apply SubLoRA to LoRA fine-tuning and rank determination of PINNs for solving a class of PDEs due to their inherent smoothness. Experimental results demonstrate that our method consistently outperforms first-order and diagonal second-order baselines. 
The use of Hessian information and the submodular projection leads to better capturing of the loss landscape and more effective rank allocation under budget constraints.

Although this work focuses mainly on designing SubLoRA and applying it to PINNs, where smoothness assumptions naturally exist due to the use of second-order information, there are promising extensions to other domains. 
In large language models, for example, the prevalent use of ReLU activations, which are not second-order differentiable, makes direct Hessian computation less suitable. However, replacing ReLU with smooth activation functions may allow the application of our method in this context. 
Similar adaptations could enable extensions to vision models and multimodal architectures. We believe SubLoRA provides a general framework for effective LoRA rank determination that is broadly applicable across domains.

\newpage 
\bibliographystyle{plain}
\bibliography{main}

\begin{thebibliography}{10}

\bibitem{awais2025foundation}
Muhammad Awais, Muzammal Naseer, Salman Khan, Rao~Muhammad Anwer, Hisham Cholakkal, Mubarak Shah, Ming-Hsuan Yang, and Fahad~Shahbaz Khan.
\newblock Foundation models defining a new era in vision: a survey and outlook.
\newblock {\em IEEE Transactions on Pattern Analysis and Machine Intelligence}, 2025.

\bibitem{bershatsky2024lotr}
Daniel Bershatsky, Daria Cherniuk, Talgat Daulbaev, Aleksandr Mikhalev, and Ivan Oseledets.
\newblock {LoTR}: Low tensor rank weight adaptation.
\newblock {\em arXiv preprint arXiv:2402.01376}, 2024.

\bibitem{bonfanti2024challenges}
Andrea Bonfanti, Giuseppe Bruno, and Cristina Cipriani.
\newblock The challenges of the nonlinear regime for physics-informed neural networks.
\newblock {\em Advances in Neural Information Processing Systems}, 37:41852--41881, 2024.

\bibitem{chen2022adaptformer}
Shoufa Chen, Chongjian Ge, Zhan Tong, Jiangliu Wang, Yibing Song, Jue Wang, and Ping Luo.
\newblock Adapt{F}ormer: Adapting vision {T}ransformers for scalable visual recognition.
\newblock {\em Advances in Neural Information Processing Systems}, 35:16664--16678, 2022.

\bibitem{dettmers2023qlora}
Tim Dettmers, Artidoro Pagnoni, Ari Holtzman, and Luke Zettlemoyer.
\newblock {QLoRA}: Efficient finetuning of quantized {LLM}s.
\newblock {\em Advances in Neural Information Processing Systems}, 36:10088--10115, 2023.

\bibitem{ding2024lora}
Chuntao Ding, Xu~Cao, Jianhang Xie, Linlin Fan, Shangguang Wang, and Zhichao Lu.
\newblock {LoRA-C}: Parameter-efficient fine-tuning of robust {CNN} for {IoT} devices.
\newblock {\em arXiv preprint arXiv:2410.16954}, 2024.

\bibitem{ding2023parameter}
Ning Ding, Yujia Qin, Guang Yang, Fuchao Wei, Zonghan Yang, Yusheng Su, Shengding Hu, Yulin Chen, Chi-Min Chan, Weize Chen, et~al.
\newblock Parameter-efficient fine-tuning of large-scale pre-trained language models.
\newblock {\em Nature Machine Intelligence}, 5(3):220--235, 2023.

\bibitem{edalati2022krona}
Ali Edalati, Marzieh Tahaei, Ivan Kobyzev, Vahid~Partovi Nia, James~J Clark, and Mehdi Rezagholizadeh.
\newblock Kron{A}: Parameter efficient tuning with {K}ronecker adapter.
\newblock {\em arXiv preprint arXiv:2212.10650}, 2022.

\bibitem{frenkel2024implicit}
Yarden Frenkel, Yael Vinker, Ariel Shamir, and Daniel Cohen-Or.
\newblock Implicit style-content separation using {B}-{LoRA}.
\newblock In {\em European Conference on Computer Vision}, pages 181--198. Springer, 2024.

\bibitem{fujishige2005submodular}
Satoru Fujishige.
\newblock {\em Submodular functions and optimization}, volume~58.
\newblock Elsevier, 2005.

\bibitem{gao2025low}
Yihang Gao, Michael~K Ng, and Vincent~YF Tan.
\newblock Low tensor-rank adaptation of kolmogorov--arnold networks.
\newblock {\em arXiv preprint arXiv:2502.06153}, 2025.

\bibitem{ge2015escaping}
Rong Ge, Furong Huang, Chi Jin, and Yang Yuan.
\newblock Escaping from saddle points—online stochastic gradient for tensor decomposition.
\newblock In {\em Conference on Learning Theory}, pages 797--842. PMLR, 2015.

\bibitem{ge2016matrix}
Rong Ge, Jason~D Lee, and Tengyu Ma.
\newblock Matrix completion has no spurious local minimum.
\newblock {\em Advances in Neural Information Processing Systems}, 29, 2016.

\bibitem{han2024parameterefficient}
Zeyu Han, Chao Gao, Jinyang Liu, Jeff Zhang, and Sai~Qian Zhang.
\newblock Parameter-efficient fine-tuning for large models: A comprehensive survey.
\newblock {\em Transactions on Machine Learning Research}, 2024.

\bibitem{hayou2024lora+}
Soufiane Hayou, Nikhil Ghosh, and Bin Yu.
\newblock {LoRA+}: Efficient low rank adaptation of large models.
\newblock In {\em International Conference on Machine Learning}, pages 17783--17806. PMLR, 2024.

\bibitem{hu2022lora}
Edward~J Hu, yelong shen, Phillip Wallis, Zeyuan Allen-Zhu, Yuanzhi Li, Shean Wang, Lu~Wang, and Weizhu Chen.
\newblock Lo{RA}: Low-rank adaptation of large language models.
\newblock In {\em International Conference on Learning Representations}, 2022.

\bibitem{jang2024lora}
Uijeong Jang, Jason~D. Lee, and Ernest~K. Ryu.
\newblock Lo{RA} training in the {NTK} regime has no spurious local minima.
\newblock In {\em Forty-first International Conference on Machine Learning}, 2024.

\bibitem{jia2022visual}
Menglin Jia, Luming Tang, Bor-Chun Chen, Claire Cardie, Serge Belongie, Bharath Hariharan, and Ser-Nam Lim.
\newblock Visual prompt tuning.
\newblock In {\em European conference on computer vision}, pages 709--727. Springer, 2022.

\bibitem{karniadakis2021physics}
George~Em Karniadakis, Ioannis~G Kevrekidis, Lu~Lu, Paris Perdikaris, Sifan Wang, and Liu Yang.
\newblock Physics-informed machine learning.
\newblock {\em Nature Reviews Physics}, 3(6):422--440, 2021.

\bibitem{kim2025lora}
Junsu Kim, Jaeyeon Kim, and Ernest~K. Ryu.
\newblock Lo{RA} training provably converges to a low-rank global minimum or it fails loudly (but it probably won't fail).
\newblock In {\em Forty-second International Conference on Machine Learning}, 2025.

\bibitem{krause2014submodular}
Andreas Krause and Daniel Golovin.
\newblock Submodular function maximization.
\newblock {\em Tractability}, 3(71-104):3, 2014.

\bibitem{liang2025lorasculpt}
Jian Liang, Wenke Huang, Guancheng Wan, Qu~Yang, and Mang Ye.
\newblock Lorasculpt: Sculpting lora for harmonizing general and specialized knowledge in multimodal large language models.
\newblock In {\em Proceedings of the Computer Vision and Pattern Recognition Conference}, pages 26170--26180, 2025.

\bibitem{liu2024alora}
Zequan Liu, Jiawen Lyn, Wei Zhu, Xing Tian, and Yvette Graham.
\newblock {AL}o{RA}: Allocating low-rank adaptation for fine-tuning large language models.
\newblock In {\em Proceedings of the 2024 Conference of the North American Chapter of the Association for Computational Linguistics: Human Language Technologies (Volume 1: Long Papers)}, pages 622--641, Mexico City, Mexico, June 2024. Association for Computational Linguistics.

\bibitem{majumdar2023hyperlora}
Ritam Majumdar, Vishal Jadhav, Anirudh Deodhar, Shirish Karande, Lovekesh Vig, and Venkataramana Runkana.
\newblock Hyper{LoRA} for {PDE}s.
\newblock {\em arXiv preprint arXiv:2308.09290}, 2023.

\bibitem{majumdar2023pihlora}
Ritam Majumdar, Vishal Jadhav, Anirudh Deodhar, Shirish Karande, Lovekesh Vig, and Venkataramana Runkana.
\newblock {PIHLoRA}: Physics-informed hypernetworks for low-ranked adaptation.
\newblock In {\em AI for Accelerated Materials Design-NeurIPS 2023 Workshop}, 2023.

\bibitem{malladi2023kernel}
Sadhika Malladi, Alexander Wettig, Dingli Yu, Danqi Chen, and Sanjeev Arora.
\newblock A kernel-based view of language model fine-tuning.
\newblock In {\em International Conference on Machine Learning}, pages 23610--23641. PMLR, 2023.

\bibitem{mao2025survey}
Yuren Mao, Yuhang Ge, Yijiang Fan, Wenyi Xu, Yu~Mi, Zhonghao Hu, and Yunjun Gao.
\newblock A survey on {LoRA} of large language models.
\newblock {\em Frontiers of Computer Science}, 19(7):197605, 2025.

\bibitem{mckay2025nearoptimal}
Maricela~Best Mckay, Avleen Kaur, Chen Greif, and Brian Wetton.
\newblock Near-optimal sketchy natural gradients for physics-informed neural networks.
\newblock In {\em Forty-second International Conference on Machine Learning}, 2025.

\bibitem{muller2023achieving}
Johannes M{\"u}ller and Marius Zeinhofer.
\newblock Achieving high accuracy with {PINN}s via energy natural gradient descent.
\newblock In {\em International Conference on Machine Learning}, pages 25471--25485. PMLR, 2023.

\bibitem{raissi2019physics}
Maziar Raissi, Paris Perdikaris, and George~E Karniadakis.
\newblock Physics-informed neural networks: A deep learning framework for solving forward and inverse problems involving nonlinear partial differential equations.
\newblock {\em Journal of Computational physics}, 378:686--707, 2019.

\bibitem{schmirler2024fine}
Robert Schmirler, Michael Heinzinger, and Burkhard Rost.
\newblock Fine-tuning protein language models boosts predictions across diverse tasks.
\newblock {\em Nature Communications}, 15(1):7407, 2024.

\bibitem{shuttleworth2024lora}
Reece Shuttleworth, Jacob Andreas, Antonio Torralba, and Pratyusha Sharma.
\newblock {LoRA} vs full fine-tuning: An illusion of equivalence.
\newblock {\em arXiv preprint arXiv:2410.21228}, 2024.

\bibitem{tjandra2018tensor}
Andros Tjandra, Sakriani Sakti, and Satoshi Nakamura.
\newblock Tensor decomposition for compressing recurrent neural network.
\newblock In {\em 2018 International Joint Conference on Neural Networks (IJCNN)}, pages 1--8. IEEE, 2018.

\bibitem{valipour2023dylora}
Mojtaba Valipour, Mehdi Rezagholizadeh, Ivan Kobyzev, and Ali Ghodsi.
\newblock {D}y{L}o{RA}: Parameter-efficient tuning of pre-trained models using dynamic search-free low-rank adaptation.
\newblock In {\em Proceedings of the 17th Conference of the European Chapter of the Association for Computational Linguistics}, pages 3274--3287, Dubrovnik, Croatia, May 2023. Association for Computational Linguistics.

\bibitem{wang2025parameter}
Luping Wang, Sheng Chen, Linnan Jiang, Shu Pan, Runze Cai, Sen Yang, and Fei Yang.
\newblock Parameter-efficient fine-tuning in large language models: A survey of methodologies.
\newblock {\em Artificial Intelligence Review}, 58(8):227, 2025.

\bibitem{wang2024blob}
Yibin Wang, Haizhou Shi, Ligong Han, Dimitris~N. Metaxas, and Hao Wang.
\newblock {BL}ob: Bayesian low-rank adaptation by backpropagation for large language models.
\newblock In {\em The Thirty-eighth Annual Conference on Neural Information Processing Systems}, 2024.

\bibitem{wang2025transfer}
Yizheng Wang, Jinshuai Bai, Mohammad~Sadegh Eshaghi, Cosmin Anitescu, Xiaoying Zhuang, Timon Rabczuk, and Yinghua Liu.
\newblock Transfer learning in physics-informed neurals networks: Full fine-tuning, lightweight fine-tuning, and low-rank adaptation.
\newblock {\em International Journal of Mechanical System Dynamics}, 2025.

\bibitem{yang2024bayesian}
Adam~X. Yang, Maxime Robeyns, Xi~Wang, and Laurence Aitchison.
\newblock Bayesian low-rank adaptation for large language models.
\newblock In {\em The Twelfth International Conference on Learning Representations}, 2024.

\bibitem{yang2024loretta}
Yifan Yang, Jiajun Zhou, Ngai Wong, and Zheng Zhang.
\newblock {L}o{RETTA}: Low-rank economic tensor-train adaptation for ultra-low-parameter fine-tuning of large language models.
\newblock In {\em Proceedings of the 2024 Conference of the North American Chapter of the Association for Computational Linguistics: Human Language Technologies (Volume 1: Long Papers)}, pages 3161--3176, Mexico City, Mexico, June 2024. Association for Computational Linguistics.

\bibitem{you2022ranking}
Kaichao You, Yong Liu, Ziyang Zhang, Jianmin Wang, Michael~I Jordan, and Mingsheng Long.
\newblock Ranking and tuning pre-trained models: A new paradigm for exploiting model hubs.
\newblock {\em Journal of Machine Learning Research}, 23(209):1--47, 2022.

\bibitem{zeng2024the}
Yuchen Zeng and Kangwook Lee.
\newblock The expressive power of low-rank adaptation.
\newblock In {\em The Twelfth International Conference on Learning Representations}, 2024.

\bibitem{zhang2023adaptive}
Qingru Zhang, Minshuo Chen, Alexander Bukharin, Pengcheng He, Yu~Cheng, Weizhu Chen, and Tuo Zhao.
\newblock Adaptive budget allocation for parameter-efficient fine-tuning.
\newblock In {\em The Eleventh International Conference on Learning Representations}, 2023.

\bibitem{zhang2024personalized}
You Zhang, Jin Wang, Liang-Chih Yu, Dan Xu, and Xuejie Zhang.
\newblock Personalized {LoRA} for human-centered text understanding.
\newblock In {\em Proceedings of the AAAI Conference on Artificial Intelligence}, volume~38, pages 19588--19596, 2024.

\bibitem{zhang2024neural}
Yuanhan Zhang, Kaiyang Zhou, and Ziwei Liu.
\newblock Neural prompt search.
\newblock {\em IEEE Transactions on Pattern Analysis and Machine Intelligence}, 2024.

\end{thebibliography}

\end{document}